\documentclass[twoside,11pt]{article}

\usepackage{jmlr2e}
\usepackage{amsfonts}
\usepackage{amsmath}
\usepackage{color}
\usepackage{graphicx}
\usepackage{wasysym}
\usepackage{amssymb}

\let\hat\widehat
\let\tilde\widetilde

\DeclareMathOperator*{\essinf}{ess\,inf}
\DeclareMathOperator*{\esssup}{ess\,sup}

\newcommand\cA{{\cal A}}

\newcommand\cC{{\cal C}}

\newcommand\cF{{\cal F}}
\newcommand\cG{{\cal G}}
\newcommand\cH{{\cal H}}
\newcommand\cI{{\cal I}}
\newcommand\cJ{{\cal J}}
\newcommand\cK{{\cal K}}

\newcommand\cM{{\cal M}}

\newcommand\cQ{{\cal Q}}

\newcommand\cS{{\cal S}}

\newcommand\Union{\bigcup}
\newcommand\intersect{\cap}

\newcommand\norm[1]{\left\|#1\right\|}
\newcommand\Set[1]{\left\{#1\right\}}
\renewcommand\angle{\mathop{\sf angle}}
\newcommand\tube{\mathop{\sf tube}}

\newenvironment{enum}{
\begin{enumerate}
  \setlength{\itemsep}{1pt}
  \setlength{\parskip}{0pt}
  \setlength{\parsep}{0pt}
}{\end{enumerate}}

\ShortHeadings{Minimax Manifold Estimation}{Genovese, Perone-Pacifico, Verdinelli and Wasserman}
\firstpageno{1}

\begin{document}

\title{Minimax Manifold Estimation\\ \today}
\author{\name Christopher R.~Genovese \email genovese@stat.cmu.edu\\
       \addr Department of Statistics\\
       Carnegie Mellon University\\
       Pittsburgh, PA 15213, USA
       \AND
       \name Marco Perone-Pacifico \email marco.peronepacifico@uniroma1.it \\
       \addr Department of Statistical Sciences\\
       Sapienza University of Rome\\
       Rome, Italy
       \AND
       \name Isabella Verdinelli \email isabella@stat.cmu.edu \\
       \addr Department of Statistics\\
       Carnegie Mellon University \\
       Pittsburgh, PA 15213, USA\\
       and Department of Statistical Sciences\\
       Sapienza University of Rome\\
       Rome, Italy
       \AND
       \name Larry Wasserman \email larry@stat.cmu.edu \\
       \addr Department of Statistics\\
       and Machine Learning Department\\
       Carnegie Mellon University\\
       Pittsburgh, PA 15213, USA
       }

\editor{}

\maketitle

\begin{abstract}%
We find the minimax rate of convergence 
in Hausdorff distance
for
estimating a manifold $M$ of dimension $d$
embedded in $\mathbb{R}^D$
given a noisy sample from the manifold.
Under certain conditions,
we show that the optimal rate of convergence is
$n^{-2/(2+d)}$.
Thus, the minimax rate depends only
on the dimension of the manifold, not on the
dimension of the space in which $M$ is embedded.
\end{abstract}

\begin{keywords}
Manifold learning, Minimax estimation.
\end{keywords}

\section{Introduction}

We consider the problem of estimating a manifold $M$
given noisy observations near the manifold.
The observed data are a random sample
$Y_1,\ldots, Y_n$ where
$Y_i\in\mathbb{R}^D$.
The model for the data is
\begin{equation}\label{eq::first}
Y_i = \xi_i +  Z_i
\end{equation}
where $\xi_1,\ldots, \xi_n$ are 
unobserved variables drawn from
a distribution supported on 
a manifold $M$
with dimension $d < D$.
The noise variables
$Z_1, \ldots, Z_n$ are 
drawn from a distribution $F$.
Our main assumption is that 
$M$ is a compact, $d$-dimensional, smooth Riemannian submanifold in $\mathbb{R}^D$;
the precise conditions on $M$ are given in
Section \ref{sec::conditions}.

A manifold $M$ and a distribution for $(\xi,Z)$
induce a distribution $Q\equiv Q_{M}$ for $Y$.
In Section 
\ref{section::dist},
we define a class of such distributions
\begin{equation}
{\cal Q}=
\Bigl\{ Q_{M} :\ M\in {\cal M}\Bigr\}
\end{equation}
where ${\cal M}$ is a set of manifolds.
Given two sets $A$ and $B$, the Hausdorff distance between $A$ and $B$ is 
\begin{equation}
H(A,B) = \inf \Bigl\{ \epsilon: \ A\subset B\oplus\epsilon
\ \ \ {\rm and}\ \ \ \ B\subset A\oplus\epsilon \Bigr\}
\end{equation}
where
\begin{equation}
A\oplus\epsilon = \bigcup_{x\in A} B_D(x,\epsilon)
\end{equation}
and
$B_D(x,\epsilon)$ is an open ball in $\mathbb{R}^D$ centered at $x$ with radius $\epsilon$.
We are interested in the 
minimax risk
\begin{equation}
R_n({\cal Q}) = \inf_{\hat M}\sup_{Q\in {\cal Q}}\mathbb{E}_Q [ H(\hat M,M)]
\end{equation}
where the infimum is over all estimators $\hat M$.
By an estimator $\hat M$ we mean a measurable function
of $Y_1,\ldots, Y_n$ taking values in the set of all manifolds.
Our first main result is the following minimax lower bound
which is proved in Section \ref{sec::minimax}.

\begin{theorem}
\label{thm::minimax}
Under the conditions given in Section 2,
there is a constant $C_1>0$ such that, for all large $n$,
\begin{equation}
\inf_{\hat M}
\sup_{Q\in {\cal Q}}
\mathbb{E}_{Q}\left[ H(\hat M,M)\right] \geq
C_1 \, \left(\frac{1}{n}\right)^{\frac{2}{2+d}}
\end{equation}
where the infimum is over all estimators $\hat M$.
\end{theorem}

Thus, no method of estimating $M$ can have an expected Hausdorff distance smaller than
the stated bound.
Note that the rate depends on $d$ but not on $D$
even though the support of the distribution $Q$ for $Y$
has dimension $D$.
Our second result is the following upper bound
which is proved in Section \ref{OUTLINE}. 

\begin{theorem}
\label{thm::upper}
Under the conditions given in Section 2,
there exists an estimator $\hat M$ such that, for all large $n$,
\begin{equation}
\sup_{Q\in{\cal Q}} \mathbb{E}_{Q}\left[ H(\hat M,M)\right] \leq
C_2 \,  \left(\frac{\log n}{n}\right)^{\frac{2}{2+d}}
\end{equation}
for some $C_2>0$.
\end{theorem}

Thus the rate is tight,
up to logarithmic factors.
The estimator in Theorem \ref{thm::upper}
is of theoretical interest because it
establishes that the lower bound is tight.
But, the estimator constructed in the proof of that
theorem is not practical and so
in Section \ref{sec::suboptimal},
we construct a very simple estimator $\hat M$ such that
\begin{equation}
\sup_{Q\in{\cal Q}} \mathbb{E}_{Q}\left[ H(\hat M,M)\right] \leq
\left(\frac{C\log n}{n}\right)^{1/D}.
\end{equation}
This is slower than the minimax rate, but 
the estimator is computationally very simple
and requires no knowledge of $d$ or 
the smoothness of $M$.

\vspace{.5cm}

{\em Related Work.}
There is a vast literature on manifold estimation.
Much of the literature deals with
using manifolds for the purpose of dimension reduction.
See, for example,
\cite{baraniuk} and references therein.
We are interested instead in actually estimating the manifold itself.
There is a large literature on this problem in
the field of computational geometry;
see, for example, \cite{Dey},
\cite{deygoswami},
\cite{chazal2008}
\cite{chengdey} and
\cite{boissonnatghosh}.
However, very few papers allow for noise in the statistical sense,
by which we mean observations drawn randomly
from a distribution.
In the literature on computational geometry,
observations are called noisy if they depart from the underlying
manifold in a very specific way:
the observations have to be close to the manifold but not too close
to each other.
This notion of noise is quite different from random sampling
from a distribution.
An exception is \cite{NSW2008}
who constructed the following estimator.
Let
$I = \{ i:\ \hat{p}(Y_i) > \lambda\}$ 
where
$\hat{p}$ is a density estimator.
They define
$\hat M = \bigcup_{i\in I} B_D(Y_i,\epsilon)$ and they show that
if $\lambda$ and $\epsilon$ are chosen properly,
then $\hat M$ is homologous to $M$.
(This means that $M$ and $\hat M$ share certain topological properties.)
However, the result does not guarantee closeness in Hausdorff distance.
Note that
$\bigcup_{i=1}^n B_D(Y_i,\epsilon)$
is precisely the
Devroye-Wise estimator for the
support of a distribution
(\cite{dw}).

\vspace{.5cm}
{\em Notation.}
Given a set $S$, we denote its boundary by
$\partial S$.
We let
$B_D(x,r)$
denote a $D$-dimensional open ball centered at $x$ with radius $r$.
If $A$ is a set and $x$ is a point then we write
$d(x,A) = \inf_{y\in A}||x-y||$ where
$||\,\cdot\,||$ is the Euclidean norm.
Let
\begin{equation}
A\circ B = (A\cap B^c)\, \bigcup \,(A^c \cap B)
\end{equation}
denote symmetric set difference between sets $A$ and $B$.

The uniform measure on a manifold $M$
is denoted by $\mu_M$.
Lebesgue measure on $\mathbb{R}^k$ is denoted by $\nu_k$.
In case $k=D$, we sometimes write $V$ instead of $\nu_D$;
in other words $V(A)$ is simply the volume of $A$.
Any integral of the form $\int f$ is understood
to be the integral with respect to Lebesgue measure on $\mathbb{R}^D$.
If $P$ and $Q$ are two probability measures 
on $\mathbb{R}^D$ with densities $p$ and $q$
then the {\em Hellinger distance} between $P$ and $Q$ is
\begin{equation}
h(P,Q) \equiv h(p,q) =
\sqrt{\int (\sqrt{p}- \sqrt{q})^2} =
\sqrt{2\left( 1 - \int \sqrt{pq}\right)}
\end{equation}
where the integrals are with respect to $\nu_D$.
Recall that 
\begin{equation}\label{eq::hell-l1}
\ell_1(p,q)\leq h(p,q) \leq \sqrt{\ell_1(p,q)}
\end{equation}
where
$\ell_1(p,q)=\int |p-q|$.
Let $p(x)\wedge q(x) = \min\{p(x),q(x)\}$.
The {\em affinity} between $P$ and $Q$ is
\begin{equation}
||P \wedge Q|| = \int p \wedge q = 1 - \frac{1}{2} \int |p-q|.
\end{equation}
Let $P^n$ denote the $n$-fold product measure
based on $n$ independent observations from $P$.
In the appendix Section \ref{sec::app1} we show that
\begin{equation}\label{eq::affinity-product}
||P^n\wedge Q^n|| \geq \frac{1}{8}\left(1 - \frac{1}{2}\int|p-q| \right)^{2n}.
\end{equation}
We write $X_n = O_P(a_n)$ to mean that,
for every $\epsilon>0$ there exists $C>0$ such that
$\mathbb{P}( ||X_n||/a_n > C) \leq \epsilon$ for all large $n$.
Throughout, we use
symbols like
$C, C_0,C_1, c,c_0, c_1\ldots $ to denote generic positive contants
whose value may be different in different expressions.

\section{Model Assumptions}

\subsection{Manifold Conditions}
\label{sec::conditions}

We shall be concerned with $d$-dimensional compact Riemannian submanifolds
without boundary
embedded in $\mathbb{R}^D$ with $d < D$.
(Informally, this means that $M$ looks like $\mathbb{R}^d$ in a small neighborhood
around any point in $M$.)
We assume that $M$ is contained in some compact set
${\cal K}\subset \mathbb{R}^D$.

At each $u\in M$ let $T_u M$ denote the tangent space to $M$
and let $T_u^\perp M $ be the normal space.
We can regard $T_u M$ as a $d$-dimensional hyperplane in $\mathbb{R}^D$
and we can regard
$T_u^\perp M $ as the $D-d$ dimensional hyperplane 
perpendicular to $T_u M$.
Define the 
{\em fiber of size $a$ at $u$} to be
$L_a(u)\equiv L_a(u,M)=T_u^\perp M\bigcap B_D(u,a)$.

Let $\Delta(M)$ be the largest $r$ such that 
each point in $M\oplus r$ has a unique projection onto $M$.
The quantity $\Delta(M)$ will be small
if either $M$ highly curved or
if $M$ is close to being self-intersecting.
Let ${\cal M}\equiv {\cal M}(\kappa)$ denote all $d$-dimensional manifolds
embedded in
${\cal K}$ such that $\Delta(M) \geq \kappa$.
Throughout this paper, $\kappa$ is a fixed positive constant.
The quantity $\Delta(M)$
has been rediscovered many times.
It is called the {\em condition number} in \cite{smale},
the {\em thickness} in 
\cite{gonzalez} and
the {\em reach} in \cite{federer}.

An equivalent definition of $\Delta(M)$ is the following:
$\Delta(M)$ is the largest number $r$ such that
the fibers $L_r(u)$ never intersect.
See Figure \ref{fig::CN}.
Note that if $M$ is a sphere then
$\Delta(M)$ is just the radius of the sphere and if $M$ is a linear space then
$\Delta(M)=\infty$.
Also, if $\sigma < \Delta(M)$ then
$M\oplus \sigma$ is the disjoint union of
its fibers:
\begin{equation}\label{eq::fiber}
M\oplus\sigma = \bigcup_{u\in M}L_\sigma(u).
\end{equation}
Define ${\sf tube}(M,a) = \bigcup_{u\in M}L_a(u).$
Thus, if $\sigma < \Delta(M)$ then
$M\oplus\sigma = {\sf tube}(M,\sigma)$.

\begin{figure}
\begin{center}
\includegraphics[scale=.35]{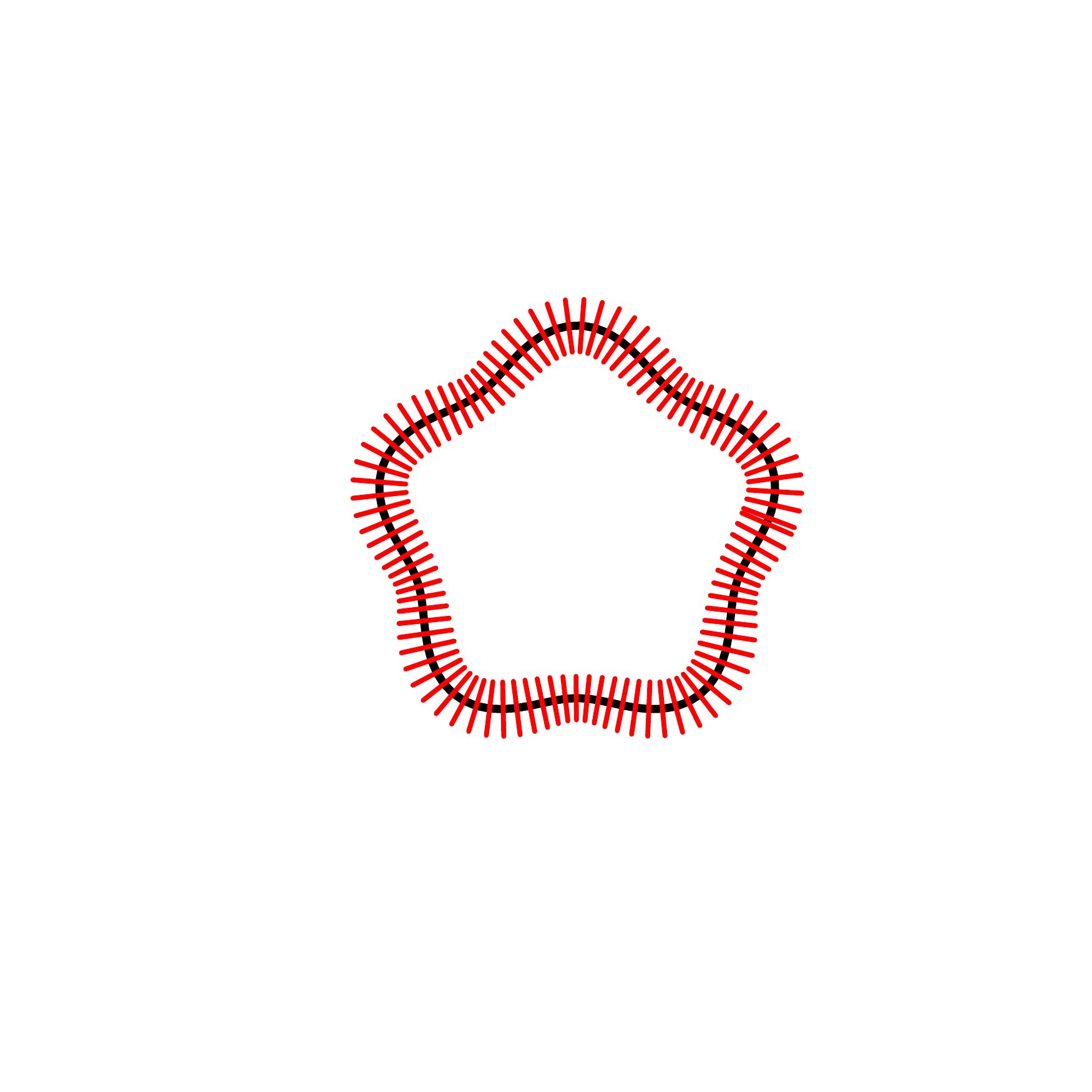}
\includegraphics[scale=.35]{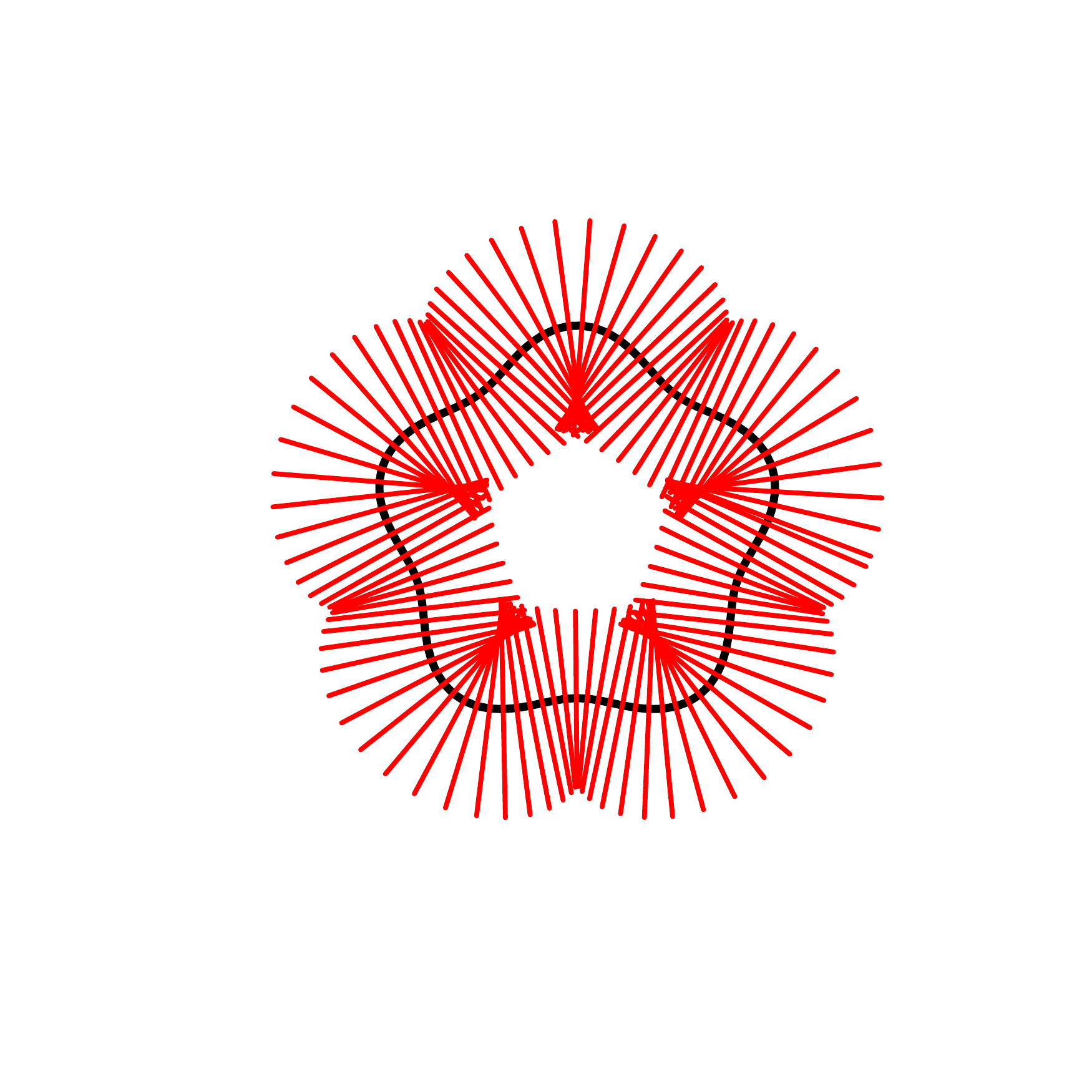}
\end{center}
\vspace{-.9in}
\caption{The condition number $\Delta(M)$ of a manifold
is the largest number $\kappa$ such that the normals to the manifold
do not cross as long as they are not extended beyond $\kappa$.
The plot on the left shows a one-dimensional manifold (a curve)
and some normals of length $r < \kappa$.
The plot on the right shows the same manifold
and some normals of length $r > \kappa$.}
\label{fig::CN}
\end{figure}

Let $p,q\in M$.
The angle between two
tangent spaces $T_p$ and $T_q$ is defined to be
\begin{equation}
{\sf angle}(T_p,T_q) = \cos^{-1}\Bigl(\min_{u\in T_p}\max_{v\in T_q}|\langle u-p, v-q\rangle|\Bigr)
\end{equation}
where
$\langle u,v\rangle$ is the usual inner product in $\mathbb{R}^D$.
Let $d_M(p,q)$ denote the geodesic distance between $p,q\in M$.

We now summarize some useful results 
from \cite{smale}.

\begin{lemma}
\label{lemma::smale}
Let $M\subset {\cal K}$ be a manifold and suppose that
$\Delta(M) = \kappa >0$.
Let $p,q\in M$.
\begin{enumerate}
\item 
Let $\gamma$ be a geodesic connecting $p$ and $q$
with unit speed parameterization.
Then the curvature of $\gamma$ is bounded above by $1/\kappa$.
\item 
$\cos(\angle(T_p,T_q)) > 1- d_M(p,q)/\kappa$.
Thus,
$\angle(T_p,T_q) \leq \sqrt{ 2 d_M(p,q)/\kappa} + o(\sqrt{ d_M(p,q)/\kappa})$.
\item If $a=||p-q|| \leq \kappa/2$ then
$d_M(p,q) \leq \kappa - \kappa \sqrt{1-(2a)/\kappa} = a + o(a)$.
\item
If $a=||p-q|| \leq \kappa/2$ then
$a\geq d_M(p,q) - (d_M(p,q))^2/(2\kappa)$.
\item If $ ||q-p|| > \epsilon$ and
$v\in B_D(q,\epsilon) \cap T_p^\perp M\cap B_D(p,\kappa)$
then $||v-p|| < \epsilon^2/\kappa$.
\item Fix any $\delta>0$.
There exists points
$x_1,\ldots, x_N \in M$ such that
$M\subset \bigcup_{j=1}^N B_D(x_j,\delta)$
and such that $N \leq (c/\delta)^d$.
\end{enumerate}
\end{lemma}

For further information about manifolds,
see \cite{Lee2002}.

\subsection{Distributional Assumptions }
\label{section::dist}

The distribution of $Y$ 
is induced by the distribution of
$\xi$ and $Z$.
We will assume that $\xi$ is drawn uniformly on the manifold.
Then we assume that $Z$ is drawn uniformly
on the normal to $M$.
More precisely,
given $\xi$, we draw $Z$
uniformly on $L_\sigma(\xi)$.
In other words, the noise is perpendicular to the manifold.
The result is that,
if $\sigma < \kappa$,
then the distribution $Q=Q_M$ of $Y$ 
has support equal to $M\oplus \sigma$.

The distributional assumption on $\xi$ is not critical.
Any smooth density bounded away from 0 on the manifold will lead to similar results.
However, the assumption on the noise $Z$ is critical.
We have chosen the simplest noise distribution here.
(Perpendicular noise is also assumed in \cite{NSW2008}.)
In current work, we are deriving the rates for more complicated noise distributions.
The rates are quite different and the proofs are more complex.
Those results will be reported elsewhere.

The set of distributions we consider is as follows.
Let $\kappa$ and $\sigma$ be fixed positive numbers
such that $0 < \sigma < \kappa$.
Let
\begin{equation}
{\cal Q} \equiv {\cal Q}(\kappa,\sigma)= \Bigl\{ Q_M: \ M\in {\cal M}(\kappa)\Bigr\}.
\end{equation}

For any
$M\in {\cal M}(\kappa)$
consider the corresponding distribution $Q_M$,
supported on $S_M=M\oplus \sigma$.
Let $q_M$ be the density of $Q_M$
with respect to Lebesgue measure.
We now show that $q_M$ is bounded above and below by a uniform density.

Recall that the essential supremum and essential infimum of $q_M$
are defined by
$$
\esssup_{y\in A} q_M =
\inf\Bigl\{ a\in \mathbb{R}:\ \nu_D(\{y:\ q_M(y) > a\}\cap A)=0\Bigr\}
$$
and
$$
\essinf_{y\in A}q_M = 
\sup\Bigl\{ a\in \mathbb{R}:\ \nu_D(\{y:\ q_M(y) < a\}\cap A)=0\Bigr\}.
$$
Also recall that, by the Lebesgue density theorem,
$q_M(y) = \lim_{\epsilon\to 0} 
Q_M(B_D(y,\epsilon))/V(B_D(y,\epsilon))$
for almost all $y$.
Let $U_M$ be the uniform distribution on $M\oplus\sigma$ and let
$u_M = 1/V(M\oplus \sigma)$ be the density of $U_M$.
Note that, for $A\subset M\oplus\sigma$,
$U_M(A) = V(A)/V(M\oplus\sigma)$.

\begin{lemma}
\label{lemma::meta}
There exist constants $0 < C_* \leq C^* < \infty$,
depending only on $\kappa$ and $d$,
such that
\begin{equation}
C_* \leq \inf_{M \in {\cal M}}\essinf_{y\in S_M} \frac{q_M(y)}{u_M (y)} \leq
\sup_{M \in {\cal M}}\esssup_{y\in S_M} \frac{q_M(y)}{u_M(y)} \leq C^*.
\end{equation}
\end{lemma}

\begin{proof}
Choose any $M\in {\cal M}(\kappa)$.
Let $x$ by any point in the interior of $S_M$.
Let $B = B_D(x,\epsilon)$ where $\epsilon>0$ is small enough
so that $B\subset S_M = M\oplus \sigma$.
Let $y$ be the projection of $x$ onto $M$.
We want to upper and lower bound
$Q(B)/V(B)$. Then we will take the limit as $\epsilon \to 0$.
Consider the two spheres of radius $\kappa$ tangent to $M$ at $y$
in the direction of the line between $x$ and $y$.
(See Figure \ref{fig::densitylemma}.)
Note that $Q(B)$ is maximized by taking $M$ to be equal to the upper sphere and
$Q(B)$ is minimized by taking $M$ to be equal to the lower sphere.
Let us consider first the case where $M$ is equal to the upper sphere.
Let
$$
U = \Bigl\{u\in M:\ L_\sigma(u)\cap B \neq \emptyset\Bigr\}
$$
be the projection of $B$ onto $M$.
By simple geometry,
$U = M \cap B_D(y,r\epsilon)$
where
$$
\left(1 + \frac{\sigma}{\kappa}\right)^{-1} \leq r \leq
\left(1 + \frac{\sigma}{\kappa}\right).
$$
Let
${\sf Vol}$ denote $d$-dimensional volume on $M$.
Then
${\sf Vol}(B_D(y,r\epsilon)\cap M) \leq c_1 r^d \epsilon^d \omega_d$
where $\omega_d$ is the volume of a unit $d$-ball and
$c_1$ depends only on $\kappa$ and $d$.
To see this, note that
because $M$ is a manifold and $\Delta(M) \ge \kappa$, it follows that near $y$, 
$M$ may be locally parameterized as
a smooth function $f = (f_1, \ldots, f_{D-d})$ over $B \intersect T_y M$.
The surface area of the graph of $f$ over $B \intersect T_y M$
is bounded by $\int_{B_D(y,r\epsilon)\intersect T_y M} \sqrt{1 + \norm{\nabla f_i}^2}$,
which is bounded by a constant $c_1$ uniformly over ${\cal M}$.
Hence, ${\sf Vol}(B_D(y,r\epsilon) \intersect M) \le 
c_1 {\sf Vol}(B_D(y,r\epsilon) \intersect T_y M) = c_1 r^d \epsilon^d\omega_d$.

Let $\Lambda_M$ be the uniform distribution on $M$ and
let $\Gamma_u$ denote the uniform measure on
$L_\sigma(u)$.
Note that, for $u\in U$,
$L_\sigma(u)\cap B$ is a $(D-d)$-ball whose radius is at most $\epsilon$.
Hence,
$$
\Gamma_u(L_\sigma(u)\cap B) \leq \frac{\epsilon^{D-d} \omega_{D-d}}{\sigma^{D-d} \omega_{D-d}}=
\left(\frac{\epsilon}{\sigma}\right)^{D-d}.
$$
Thus,
\begin{eqnarray*}
Q_M(B) &=& \int_M \Gamma_u(B\cap L_\sigma(u)) d\Lambda_M(u) =
\int_U \Gamma_u(B\cap L_\sigma(u)) d\Lambda_M(u)\\
& \leq & \left(\frac{\epsilon}{\sigma}\right)^{D-d} \Lambda(U) =
\left(\frac{\epsilon}{\sigma}\right)^{D-d} 
\frac{{\sf Vol}(B_D(y,r)\cap M)}{{\sf Vol}(M)}\\
& \leq &
\left(\frac{\epsilon}{\sigma}\right)^{D-d} 
\frac{\epsilon^d r^d \omega_d}{{\sf Vol}(M)} \leq
\left(\frac{\epsilon}{\sigma}\right)^{D-d} 
\frac{\epsilon^d (1+\sigma/\kappa)^d \omega_d}{{\sf Vol}(M)}.
\end{eqnarray*}
Now,
$U_M(B) = V(B)/V(M\oplus \sigma) = \epsilon^D \omega_D/( \sigma^{D-d}\,{\sf Vol}(M))$.
Hence,
$$
\frac{Q_M(B)}{U_M(B)} \leq
\left(1 + \frac{\sigma}{\kappa}\right)^d \omega_d.
$$
Taking limits as $\epsilon\to 0$ we have that
$q_M(y) \leq C^* u_M(y)$ for almost all $y$.

The proof of the lower bound is similar to the upper bound except for the following changes:
let $U_0$ denote all $u\in U$
such that
the radius of $B\cap L_\sigma(u)$ is at least $\epsilon/2$.
Then
$\Lambda(U_0) \geq \Lambda(U) (1-O(\epsilon))$
and the projection of $U_0$ onto $M$ is again of the form
$B_D(y,r\epsilon)\cap M$.
By Lemma 5.3 of \cite{smale},
$$
{\sf Vol}(B_D(y,r)\cap M) \geq \left(1 - \frac{r^2 \epsilon^2}{4\kappa^2}\right)^{d/2} r^d \epsilon^d\omega_d
$$
and the latter is larger than
$2^{-d/2}r^d \epsilon^d\omega_d$ for all small $\epsilon$.
Also,
$\Gamma_u(L_\sigma(u)\cap B) \geq
(\epsilon/(2\sigma))^{D-d}$ for all $u\in U_0$.
\end{proof}

Of course, an immediate consequence of the above lemma is that,
for every $M\in {\cal M}(\kappa)$ and every measurable set $A$,
$C_* \, U_M(A) \leq Q_M(A) \leq C^* \, U_M(A)$.

\begin{figure}
\vspace{-.5in}
\begin{center}
\includegraphics[scale=.7]{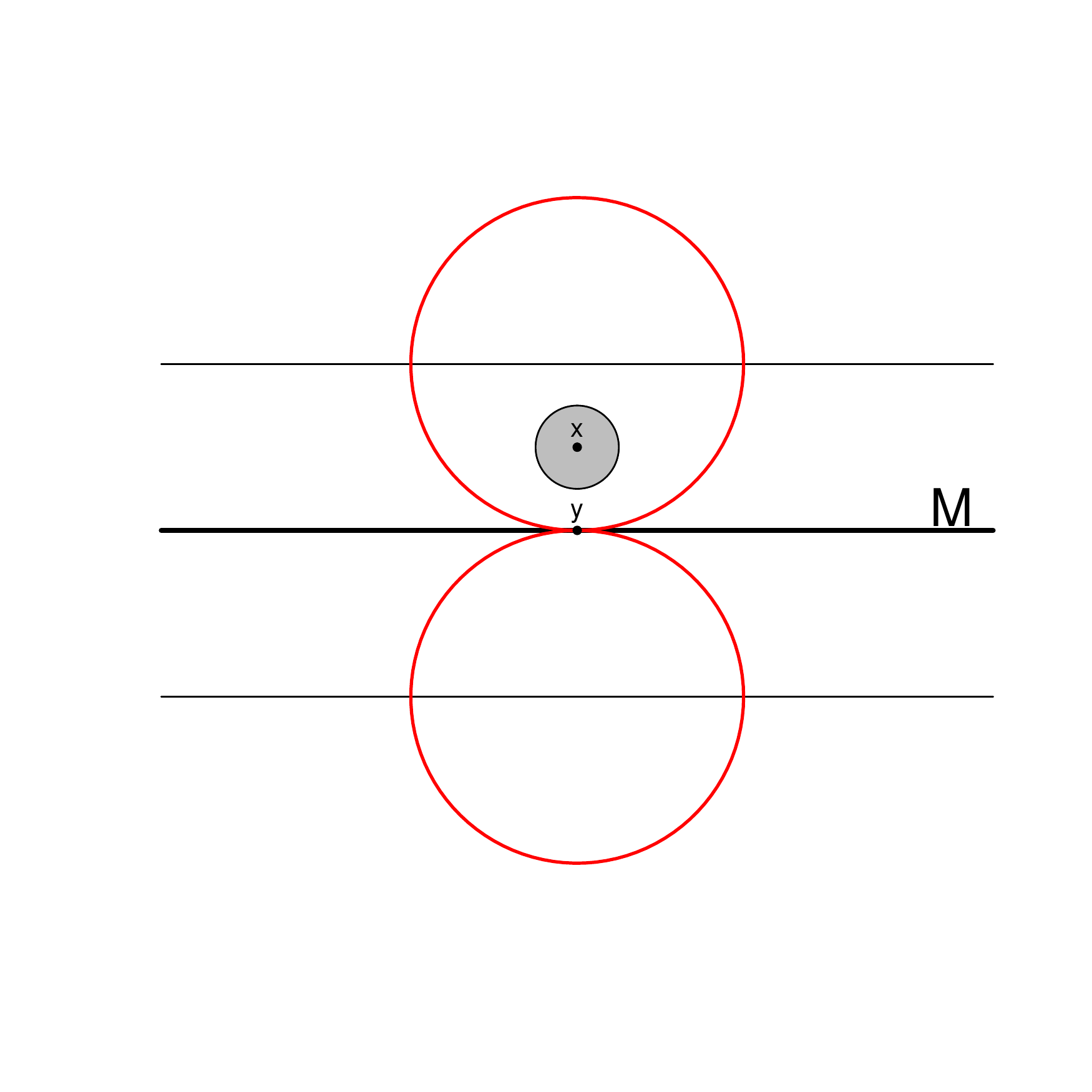}
\end{center}
\vspace{-1.4in}
\caption{Figure for proof of Lemma \ref{lemma::meta}.
$x$ is a point in the support $M\oplus \sigma$. $y$ is the projection
of $x$ onto $M$. The two spheres are tangent to $M$ at $y$ and have radius $\kappa$.}
\label{fig::densitylemma}
\end{figure}

\section{Minimax Lower Bound}
\label{sec::minimax}

In this section we derive a lower bound on the minimax rate of convergence for this problem.
We will make use of the following result due to \cite{lecam}.
The following version is from Lemma 1 of
\cite{binyu}.

\begin{lemma}[Le Cam 1973]
Let ${\cal Q}$
be a set of distributions.
Let $\theta(Q)$ take values in a metric space with metric $\rho$.
Let $Q_0,Q_1\in {\cal Q}$ be any pair of distributions in ${\cal Q}$.
Let $Y_1,\ldots, Y_n$ be drawn iid from some $Q\in {\cal Q}$
and denote the corresponding product measure by
$Q^n$.
Let $\hat\theta (Y_1,\ldots, Y_n)$ be any estimator.
Then
\begin{equation}
\sup_{Q\in {\cal Q}} 
\mathbb{E}_{Q^n}\Bigl[\rho(\hat\theta(Y_1,\ldots, Y_n),\theta(Q))\Bigr] \geq
\rho\bigl(\theta(Q_0),\theta(Q_1)\bigr) \  ||Q_0^n \wedge Q_1^n||.
\end{equation}
\end{lemma}

To get a useful bound from Le Cam's lemma, we need to construct an appropriate pair
$Q_0$ and $Q_1$.
This is the topic of the next subsection.

\subsection{A Geometric Construction }
\label{section::geometric}

In this section,
we construct a pair of manifolds 
$M_0,M_1 \in {\cal M}(\kappa)$ 
and corresponding distributions
$Q_0, Q_1$ for use in Le Cam's lemma.
An informal description 
is as follows.
Roughly speaking, $M_0$ and $M_1$ 
minimize the Hellinger distance
$h(Q_{0},Q_{1})$ subject to
their Hausdorff distance
$H(M_0,M_1)$ being equal to a given value $\gamma$.

Let 
\begin{equation}
M_0 = \Bigl\{ (u_1,\ldots, u_d,0,\ldots, 0):\ -1 \leq u_j \leq 1,\ 1\leq j \leq d\Bigr\}
\end{equation}
be a $d$-dimensional hyperplane in $\mathbb{R}^D$.
Hence $\Delta(M_0) = \infty$.
Place a hypersphere of radius $\kappa$
below $M_0$.
Push the sphere upwards into $M_0$ causing a bump of height $\gamma$
at the origin.
This creates a new manifold $M_0'$ such that
$H(M_0,M_0') = \gamma$.
However, $M_0'$ is not smooth.
We will roll a sphere of radius $\kappa$ around $M_0'$
to get a smooth manifold $M_1$ as in
Figure \ref{fig::construction}.
The formal details of the construction are in 
Section \ref{sec::flying-saucer}.

\begin{figure}
\begin{center}
\begin{tabular}{c}
\fbox{\includegraphics[width=7cm]{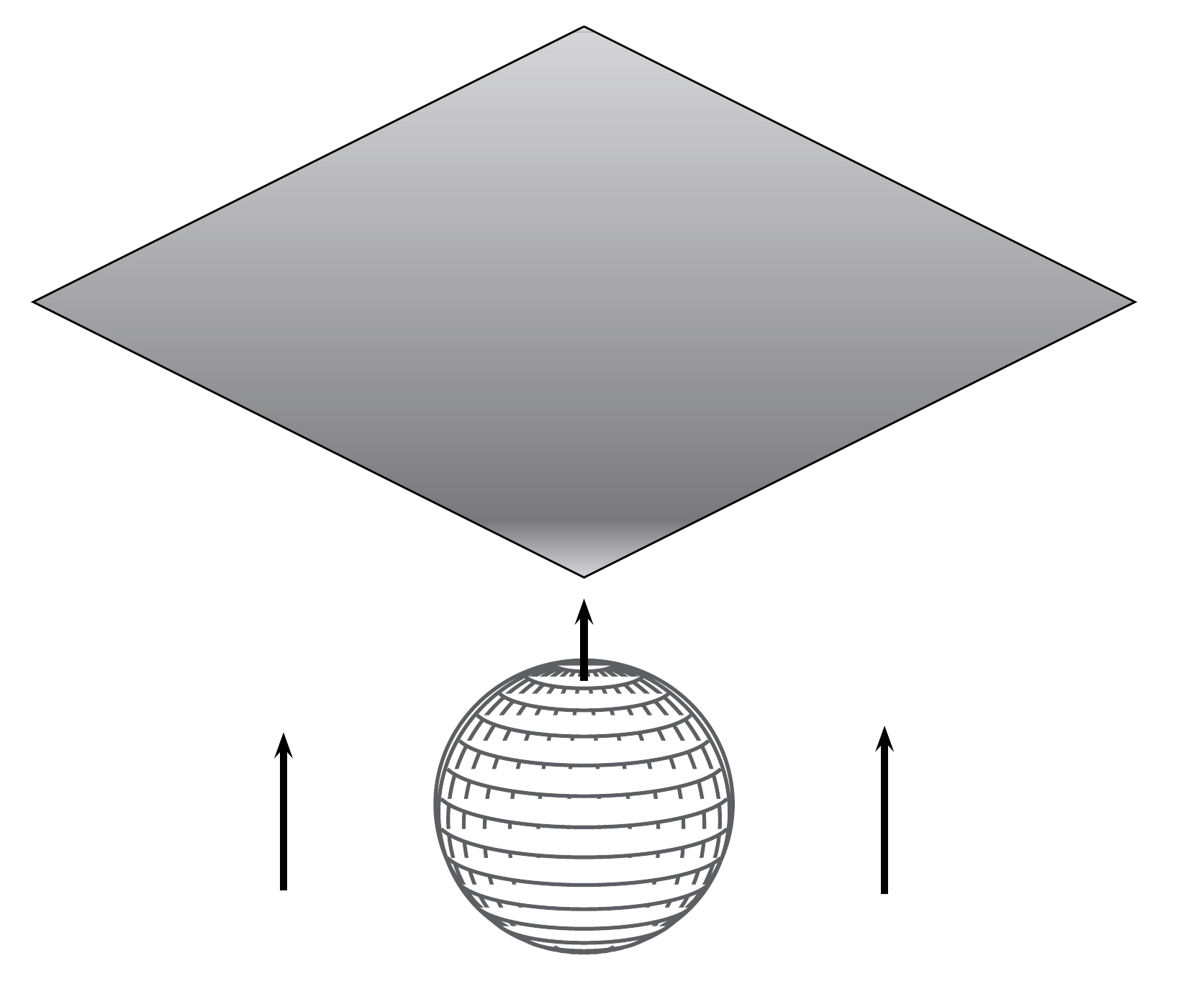}
A }\\
\\
\fbox{\includegraphics[width=7cm]{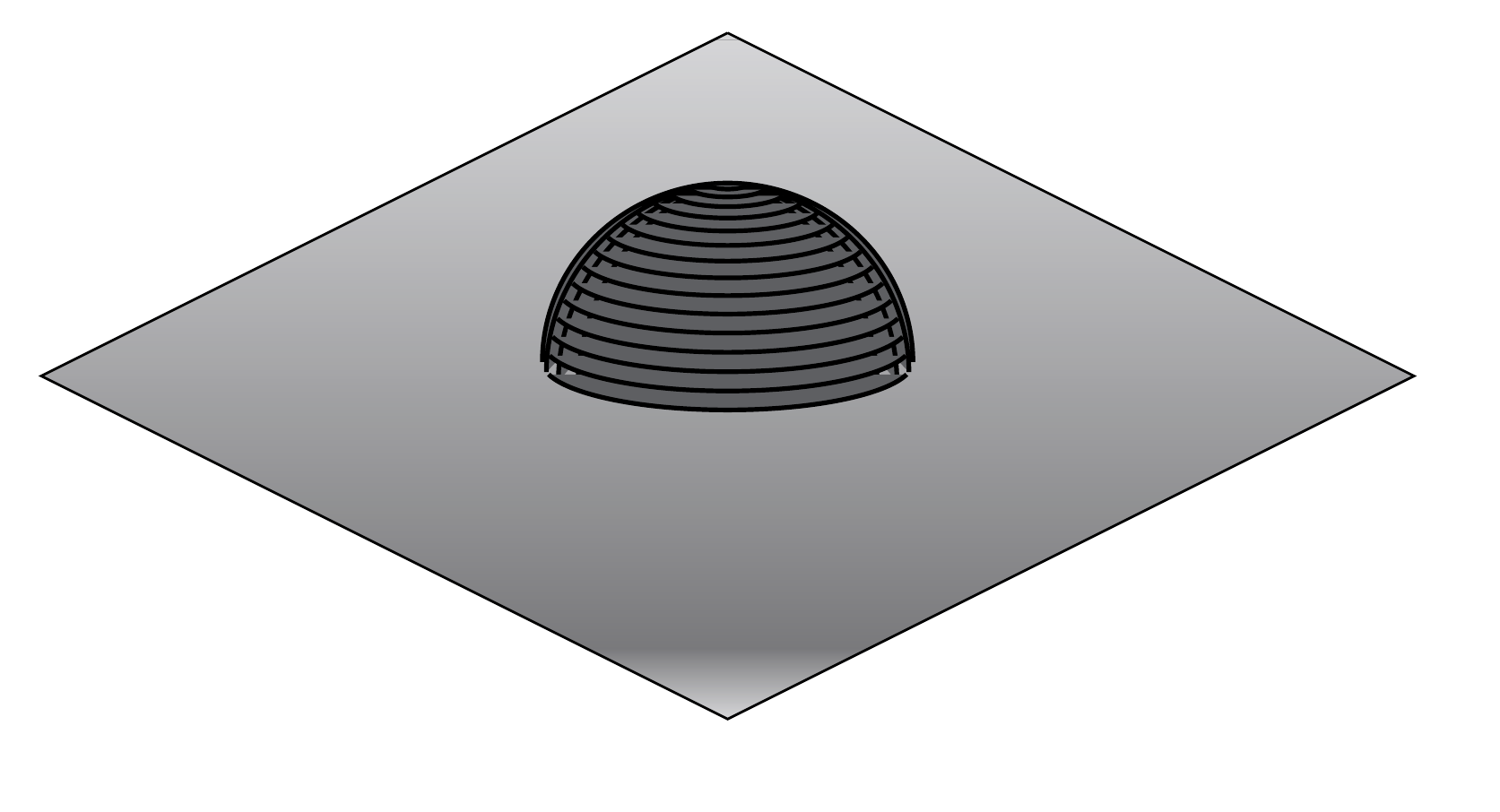} 
B }\\
\\
\fbox{\includegraphics[width=7cm]{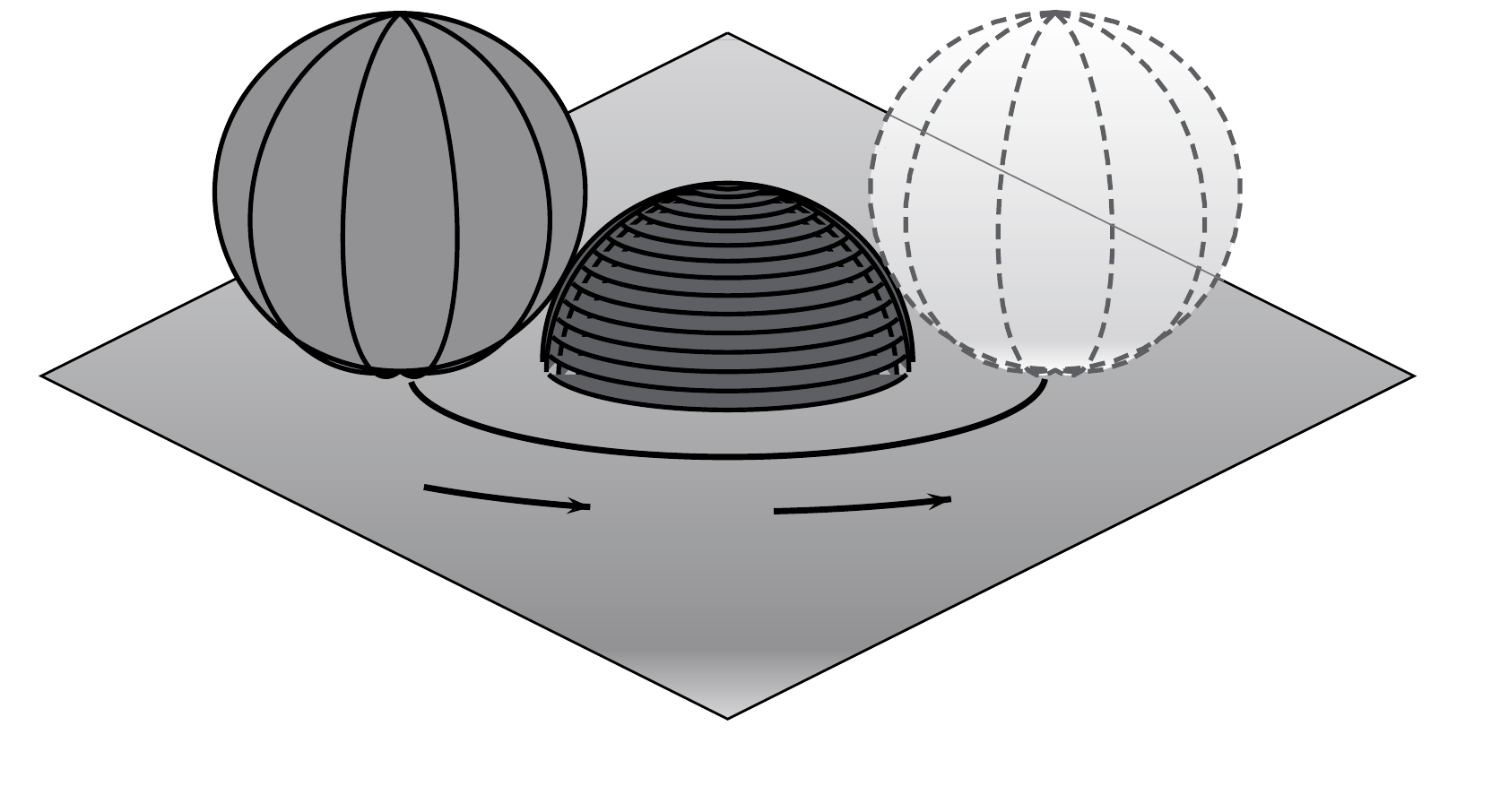} 
C }\\
\\
\fbox{\includegraphics[width=7cm]{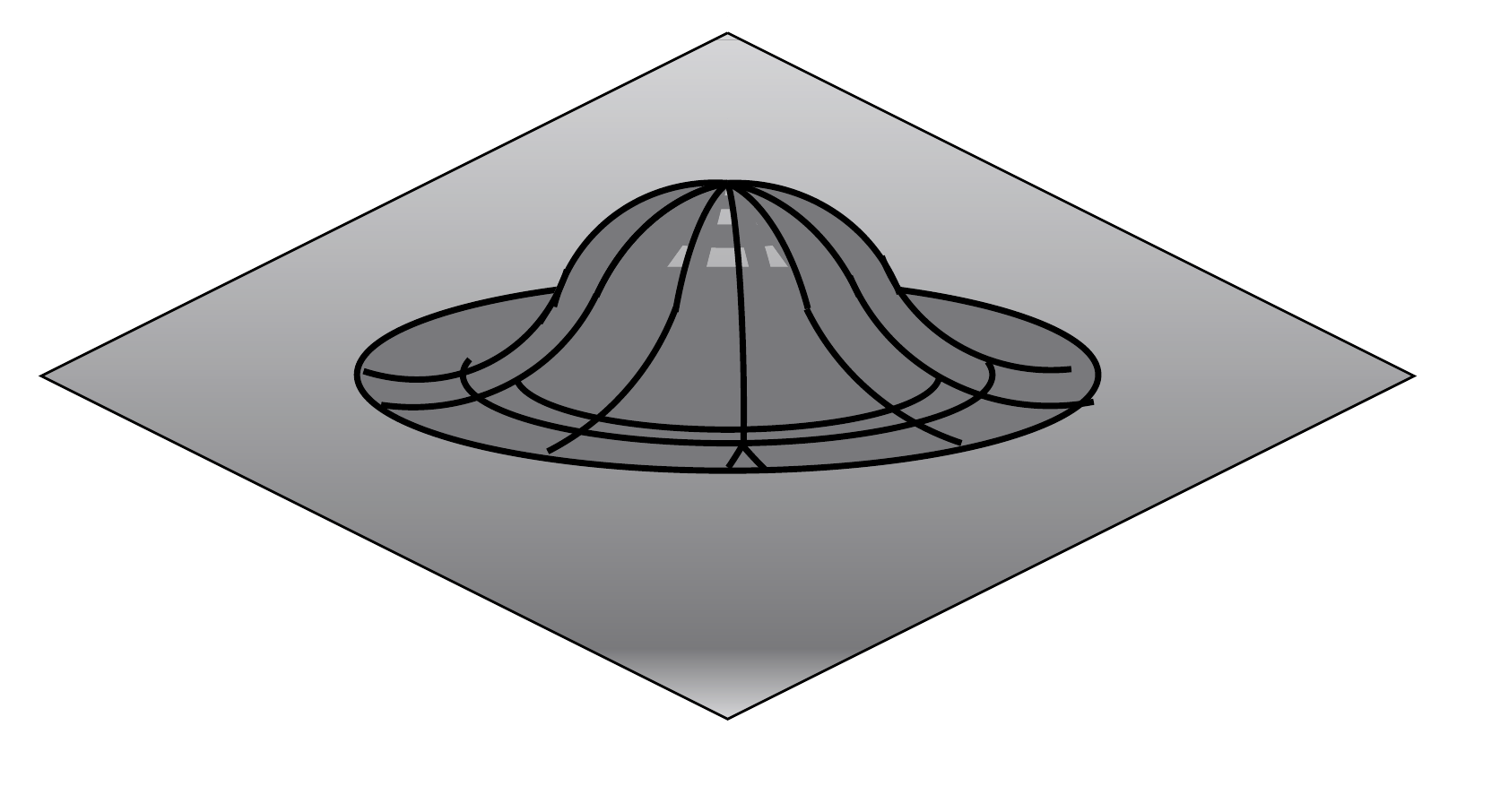} 
D }
\end{tabular}
\end{center}
\caption{A sphere of radius $\kappa$ is pushed upwards into the plane $M_0$ (panel A).
The resulting manifold $M_0'$ is not smooth (panel B).
A sphere is then rolled around the manifold (panel C) to produce a
smooth manifold $M_1$ (panel D). }
\label{fig::construction}
\end{figure}

\begin{theorem}
\label{thm::geometric}
Let $\gamma$ be a small positive number.
Let $M_0$ and $M_1$ be as defined in Section \ref{sec::flying-saucer}.
Let $Q_i$ be the corresponding distributions on $M_i\oplus \sigma$ for $i=0,1$.
Then:
\begin{enum}
\item $\Delta(M_i)\geq\kappa$, $i=0,1$.
\item $H(M_0,M_1)=\gamma$. 
\item $\int|q_0 - q_1| = O(\gamma^{(d+2)/2})$.
\end{enum}
\end{theorem}

\begin{proof}
See Section \ref{sec::flying-saucer}.
\end{proof}

\subsection{Proof of the Lower Bound}

Now we are in a position to prove the first theorem.
Let us first restate the theorem.

\vspace{1cm}

\noindent
{\bf Theorem 1.}
{\em There is a constant $C>0$ such that, for all large $n$,
\begin{equation}
\inf_{\hat M}
\sup_{Q\in {\cal Q}}
\mathbb{E}_{Q}\left[ H(\hat M,M)\right] \geq
C  n^{-\frac{2}{2+d}}
\end{equation}
where the infimum is over all estimators $\hat M$.}

\vspace{1cm}

\noindent
{\bf Proof of Theorem \ref{thm::minimax}.}
Let $M_0$ and $M_1$ be as defined
in Section \ref{section::geometric}.
Let $Q_i$ be the uniform distribution on
$M_i \oplus \sigma$, $i=0,1$.
Let $q_i$ be the density of $Q_i$ with respect to 
Lebesgue measure $\nu_D$, $i=0,1$.
Then, from Theorem 
\ref{thm::geometric},
$H(M_0,M_1) = \gamma$ and
$\int |q_0 - q_1| = O(\gamma^{(d+2)/2})$.
Le Cam's lemma then gives, for any $\hat M$,
$$
\sup_{Q\in {\cal Q}}
\mathbb{E}_{Q^n} [H (M,\hat M)] \geq
H (M_0,M_1) \ ||Q_0^n \wedge Q_1^n|| \geq 
\gamma (1-c\gamma^{(d+2)/2})^{2n}
$$
where we used equation (\ref{eq::affinity-product}).
Setting 
$\gamma = n^{-2/(d+2)}$ yields the result.
$\blacksquare$

\section{Upper bound}
\label{sec::upper-bound}

To establish the upper bound, we will construct an estimator that achieves
the appropriate rate.
The estimator is intended only for the theoretical
purpose of establishing the rate.
(A simpler but non-optimal method is discussed in Section \ref{sec::suboptimal}.)
Recall that
${\cal M}= {\cal M}(\kappa)$ is the set of
all $d$-dimensional submanifolds $M$
contained in ${\cal K}$
such that
$\Delta(M)\geq \kappa >0$.
Before proceeding, we need to discuss sieve maximum likelihood.

\vspace{1cm}

{\bf Sieve Maximum Likelihood.}
Let ${\cal P}$ be any set of distributions such that
each $P\in {\cal P}$ has a density $p$
with respect to Lebesgue measure $\nu_D$.
Recall that $h$ denotes Hellinger distance.
A set of pairs of functions
${\cal B}=\{ (\ell_1,u_1),\ldots, (\ell_N,u_N)\}$
is an $\epsilon$-Hellinger bracketing for ${\cal P}$ if,
(i) for each $p\in {\cal P}$ there is a
$(\ell,u)\in {\cal B}$ such that
$\ell(y) \leq p(y) \leq u(y)$ for all $y$ and
(ii) $h(\ell,u) \leq \epsilon$.
The logarithm of the size of the smallest $\epsilon$-bracketing
is called the {\em bracketing entropy} and is denoted by
$\cH_{[\,]}(\epsilon,{\cal P},h)$.

We will make use of the following result
which is Example 4 of \cite{shen}.

\begin{theorem}[\cite{shen}]
\label{thm::shen}
Let $\epsilon_n$ solve the equation
$\cH_{[\,]}(\epsilon_n,{\cal P},h) = n \epsilon_n^2$.
Let
$(\ell_1,u_1),\ldots, (\ell_N,u_N)$
be an $\epsilon_n$ bracketing where
$N=\cH_{[\,]}(\epsilon_n,{\cal P},h)$.
Define the set of densities
$S^*_n = \{p_1^*,\ldots, p_N^*\}$
where $p_t^* = u_t/\int u_t$.
Let $\hat p^*$ maximize
the likelihood 
$\prod_{i=1}^n p_t^*(Y_i)$
over the set $S^*_n$.
Then
\begin{equation}
\sup_{P\in {\cal P}}
P^n\left(\{ h(p,\hat p^*) \geq \epsilon_n\}\right) \leq c_1 e^{- c_2 n \epsilon_n^2}.
\end{equation}
\end{theorem}

\vspace{.5cm}

The sequence 
$\{S_n^*\}$ in
Theorem \ref{thm::shen} is called a {\em sieve}
and the estimator 
$\hat p^*$ is called a {\em sieve-maximum likelihood estimator}.
The estimator $\hat p^*$ need not be in ${\cal P}$.
We will actually need an estimator that is contained in ${\cal P}$.
We may construct one as follows.
Let
$\hat p^*$ be the sieve mle
corresponding to $S_n^*$.
Then $\hat p^* = p_t^*$ for some $t$.
Let 
$(\hat\ell,\hat u)\equiv (\ell_t,u_t)$ be the corresponding bracket.

\begin{lemma}
\label{lemma::shen2}
Assume the conditions in
Theorem \ref{thm::shen}.
Let $\hat p$ be any density in ${\cal P}$ such that
$\hat\ell \leq \hat p \leq \hat u$.
If $\epsilon_n \leq 1$ then
\begin{equation}
\sup_{P\in {\cal P}}
P^n\left(\{ h(p,\hat p) \geq c\epsilon_n \}\right) \leq c_1 e^{- c_2 n \epsilon_n^2}.
\end{equation}
\end{lemma}

\begin{proof}
By the triangle inequality,
$h(p,\hat p) \leq h(p,\hat{p}^*) + h(\hat p,\hat p^*) =
h(p,\hat{p}^*) + h(\hat p, u_t/\int u_t)$
where $\hat p^* = u_t/\int u_t$ for some $t$.
From Theorem \ref{thm::shen},
$h(p,\hat{p}^*) \leq \epsilon_n$ with high probability.
Thus we need to show that
$h(\hat p, u_t/\int u_t) \leq C \epsilon_n$.
It suffices to show that, in general,
$h(p, u/\int u) \leq C\, h(\ell,u)$
whenever $\ell \leq p \leq u$.

Let $(\ell,u)$ be a bracket and let
$\delta^2 = h^2(\ell,u)\leq 1$.
Let $\ell \leq p \leq u$.
We claim that
$h^2(p,u/\int u) \leq 4 \delta^2$.
(Taking $\delta = \epsilon_n$ then proves the result.)
Let $c^2 = \int u$.
Then
$1\leq c^2 = \int u = \int p + \int(u-p) = 1+ \int(u-p) =
1 + \ell_1(u,p) \leq 
1 +  2h(u,\ell) = 1+2\delta$.
Now,
\begin{eqnarray*}
h^2\left(p,\frac{u}{\int u}\right) &=&
\int ( \sqrt{u}/c - \sqrt{p})^2 =
\frac{1}{c^2} \int (\sqrt{u}- c\sqrt{p})^2 \leq
\int (\sqrt{u}- c\sqrt{p})^2\\
&=&
\int ( (\sqrt{u} - \sqrt{p}) + (c-1)\sqrt{p})^2 \leq
2\int  (\sqrt{u} - \sqrt{p})^2 + 2(c-1)^2\\
& \leq &
2\delta^2 + 2 (\sqrt{1+2\delta}-1)^2 \leq 2\delta^2 + 2\delta^2 = 4\delta^2
\end{eqnarray*}
where the last inequality used the fact that $\delta \leq 1$.
\end{proof}

In light of the above result,
we define modified
maximum likelihood sieve estimator $\hat p$ to be any $p\in {\cal P}$ such that
$\hat\ell \leq \hat p \leq \hat u$.
For simplicity, in the rest of the paper,
we refer to the modified sieve estimator $\hat p$, simply
as the maximum likelihood estimator (mle).

\vspace{.8cm}
\noindent
\fbox{\bf Outline of proof.\label{OUTLINE}}
\vspace{0.3cm}

\noindent We are now ready to find an estimator $\hat M$
that converges at the optimal rate (up to logarithmic terms.)
Our strategy for estimating $M$ has the following steps: 

\begin{enum}
\item [{\bf Step 1.}] We split the data into two halves.
\item [{\bf Step 2.}] Let $\tilde Q$ be the maximum likelihood
estimator using the first half of the data.
Define $\tilde M$ to be the corresponding manifold.
We call $\tilde M$, the pilot estimator.
We show that $\tilde M$ is a consistent estimator of
$M$ that converges at a sub-optimal rate 
$a_n = n^{-\frac{2}{D(d+2)}}$.
To show this we:
\begin{enum}
\item [{\bf a.}] Compute the Hellinger bracketing entropy of ${\cal Q}$. 
(Theorem \ref{thm::hausdorff-covering-number}, Lemmas \ref{lemma::volume-diff} and 
\ref{lemma::this-is-a-bracket}).
\item [{\bf b.}] Establish the rate of convergence of 
the mle in Hellinger distance,
using the bracketing entropy and  Theorem \ref{thm::shen}.
\item [{\bf c.}] Relate the Hausdorff distance to the Hellinger distance and hence 
establish the rate of convergence $a_n$ of the mle in Hausdorff distance.
(Lemma \ref{lemma::calibration}).
\item [{\bf d.}] Conclude that the true manifold is contained, with high probability, in
${\cal M}_n =\{M\in {\cal M}(\kappa):\ H(M,\tilde M) \leq a_n\}$ (Lemma \ref{lemma::pilot}).
Hence, we can now restrict attention to  ${\cal M}_n$.
\end{enum}
\item [{\bf Step 3.}] To improve the pilot estimator, we need to control the relationship
between Hellinger and Hausdorff distance and thus need to work over small sets on which
the manifold cannot vary too greatly.
Hence, we cover the pilot estimator with long, thin slabs
$R_1,\ldots, R_N$.
We do this by first covering $\tilde M$ with spheres
$\gimel_1,\ldots,\gimel_N$ of radius
$\delta_n = O( (\log n/n)^{1/(2+d)})$.
We define a slab $R_j$ to be the union of fibers 
of size $b=\sigma+a_n$ within one of the spheres:
$R_j = \cup_{x\in\gimel_j}L_b(x,\tilde M)$.
We then show that:
\begin{enum}
\item [{\bf a.}] The set of fibers on $\tilde M$ cover each $M\in {\cal M}_n$ in a nice
way. In particular, if $M\in {\cal M}_n$ then each fiber from $\tilde M$ is nearly 
normal to $M$. (Lemma \ref{lemma::toothpick}).
\item [{\bf b.}]  As $M$ cuts through a slab, it stays nearly parallel to $\tilde M$.  
Roughly speaking, $M$ behaves like a smooth, nearly linear function within 
each slab. (Lemma \ref{lemma::slabs}).
\end{enum}
\item [{\bf Step 4.}]  Using the second half of the data,
we apply maximum likelihood within each slab.
This defines estimators $\hat M_j$, for $1 \le j \le N$.
We show that:
\begin{enum}
\item [{\bf a.}] The entropy of the set of distributions within a slab is very small. 
(Lemma \ref{lemma::entropy-distr}).
\item [{\bf b.}] Because the entropy is small, the maximum likelihood estimator within 
a slab converges fairly quickly in Hellinger distance. The rate is 
$\epsilon_n = (\log n/n)^{1/(2+d)}$. (Lemma \ref{lemma::hellinger-in-a-slab}).
\item [{\bf c.}] Within a slab, there is a tight relationship between Hellinger 
distance and Hausdorff distance. Specifically, $H(M_1,M_2)\leq c \,h^2(Q_1,Q_2)$.
(Lemma \ref{lemma::hellinger-hausdorff-slab}).
\item [{\bf d.}] Steps (4b) and (4c)
imply that $H(M\cap R_j,\hat M_j) = O_P(\epsilon_n^2) =O_P( (\log n/n)^{2/(d+2)})$.
\end{enum}
\item [{\bf Step 5.}] Finally we define $\hat M = \bigcup_{j=1}^N \hat M_j$
and show that $\hat M$ converges at the optimal rate because each $\hat M_j$ 
does within its own slab.
\end{enum}

The reason for getting a preliminary estimator and then covering the estimator
with thin slabs is that, within a slab, there is a tight relationship
between Hellinger distance and Hausdorff distance.
This is not true globally but only in thin slabs.
Maximum likelihood is optimal with respect to Hellinger distance.
Within a slab, this allows us to get optimal rates in Hausdorff distance.

\vspace{.7cm}
\noindent
\fbox{\bf Step 1:} {\bf Data Splitting}
\vspace{.7cm}

\noindent
For simplicity assume the sample size is even and denote it by $2n$.
We split the data into two halves
which we denote by
$X=(X_1,\ldots, X_n)$ and
$Y=(Y_1,\ldots, Y_n)$.

\vspace{1.0cm}
\noindent
\fbox{\bf Step 2:} {\bf Pilot Estimator}
\vspace{0.5cm}

\noindent
Let $\tilde q$ be the
maximum likelihood estimator
over ${\cal Q}$.
Let $\tilde M$ be the corresponding manifold.
To study the properties of
$\tilde M$ requires two steps:
computing the bracketing entropy of ${\cal Q}$ 
and relating $H(M,\tilde M)$ to $h(q,\tilde q)$.
The former allows us to apply Theorem \ref{thm::shen} 
to bound $h(q,\tilde q)$,
and the latter allows us to control the Hausdorff distance.

\vspace{.5cm}

\noindent
{\bf Step 2a: Computing the Entropy of ${\cal Q}$.}
To compute the entropy of ${\cal Q}$
we start by constructing a finite net of manfolds to cover
${\cal M}(\kappa)$.
A finite set of $d$-manifolds
$\mathbb{M}_\gamma=\{M_1,\ldots, M_N\}$ is a $\gamma$-net 
(or a $\gamma$-cover)
if, for each
$M\in \cM$ there exists $M_j\in \mathbb{M}_\gamma$
such that $H(M,M_j)\leq \gamma$.
Let $N(\gamma)=N(\gamma,\cM,H)$ be the size of the smallest covering set,
called the (Hausdorff) covering number of $\cM$.

\begin{theorem}
\label{thm::hausdorff-covering-number}
The Hausdorff covering number of $\cM$ satisfies the following:
\begin{equation}
N(\gamma) \equiv N(\gamma,\cM,H) \le c_1\,\kappa_2(\kappa,d,D) \exp\left(\kappa_3(\kappa,d,D)\,\gamma^{-d/2}\right)
\equiv c \exp\left( c' \gamma^{-d/2}\right)
\end{equation}
where $\kappa_2(\kappa,d,D) = {\binom{D}{d}}^{(c_2/\kappa)^D}$ 
and $\kappa_3(\kappa,d,D) = 2^{d/2} (D-d)(c_2/\kappa)^D$,
for a constant $c_2$ that depends only on $\kappa$ and $d$.
\end{theorem}

\begin{proof}
Recall that the manifolds in $\cM$ all lie within $\cK$.
Consider any hypercube containing $\cK$.
Divide this cube into a grid of $J = (2c/\kappa)^D$ sub-cubes $\{C_1,\ldots, C_J\}$
of side length $\kappa/c$,
where $c \ge 4$ is a positive constant
chosen to be sufficiently large.
Our strategy is to show that
within each of these cubes, the manifold
is the graph of a smooth function.
We then only need count the number of such smooth functions.

In thinking about the manifold as (locally) the graph
of a smooth function, it helps to be able to translate
easily between the natural coordinates in $\cK$
and the domain-range coordinates of the function.
To that end, within each subcube $C_j$ for $j\in\Set{1,\ldots,J}$, 
we define $K = \binom{D}{d}$ coordinate frames,
$F_{jk}$ for $k\in\Set{1,\ldots,K}$,
in which $d$ out of $D$ coordinates are labeled as ``domain''
and the remaining $D - d$ coordinates are labeled as ``range.''

Each frame is associated with a relabeling of the coordinates
so that the $d$ ``domain'' coordinates are listed first
and $D-d$ ``range'' coordinates last.
That is, $F_{jk}$ is defined by a one-to-one correspondence
between $x\in C_j$ and $(u,v) \in \pi_{jk}(x)$
where $u\in\mathbb{R}^d$ and $v\in\mathbb{R}^{D-d}$
and $\pi_{jk}(x_1,\ldots,x_D) = (x_{i_1},\ldots,x_{i_d},x_{j_1},\ldots,x_{j_{D-d}})$
for domain coordinate indices $i_1 < \ldots < i_d$ and range coordinate indices
$j_1 < \ldots < j_{D-d}$.

We define ${\rm domain}(F_{jk}) = \{u\in\mathbb{R}^d:\; \exists v\in\mathbb{R}^{D-d} \ \mbox{such that}\ (u,v)\in F_{jk}\}$,
and let $\cG_{jk}$ denote the class of functions defined on ${\rm domain}(F_{jk})$
whose second derivative (i.e., second fundamental form) is bounded above
by a constant $C(\kappa)$ that depends only on $\kappa$.
To say that a set $R\subset C_j$ is the graph of a function on a $d$-dimensional
subset of the coordinates in $C_j$ is equivalent to saying
that for some frame $F_{jk}$ and some set $A\subset{\rm domain}(F_{jk})$, 
$R = \pi_{jk}^{-1}\Set{(u,f(u)):\; u\in A}$.

We will prove the theorem by establishing the following claims.

\vspace{.5cm}

\noindent
\emph{Claim 1}. 
Let $M\in\cM$ and $C_j$ be a subcube that intersects $M$.
Then:
(i)~for at least one $k\in\{1,\ldots,K\}$,
the set $M \cap C_j$ is the graph of a function (i.e., single-valued mapping)
defined on a set $\cA \subset {\rm domain}(F_{jk})$,
of the form $(u_1,\ldots,u_d)\mapsto \pi^{-1}_{jk}((u,f(u)))$ for some function $f$ on $\cA$,
and (ii)~this function lies in $\cG_{jk}$.

\vspace{.5cm}

\noindent
\emph{Claim 2}. $\cM$ is in one-to-one correspondence with a subset of
$\cG = \prod_{j=1}^J \Union_{k=1}^K \cG_{jk}$.

\vspace{.5cm}

\noindent
\emph{Claim 3}. The $L^\infty$ covering number of $\cG$ satisfies
$$
N(\gamma,\cG,L^\infty) \le c_1 \binom{D}{d}^{(2 c/\kappa)^D}\exp\left( (D-d)(2 c/\kappa)^D \gamma^{-d/2}\right).
$$

\vspace{.5cm}

\noindent
\emph{Claim 4}. There is a one-to-one correspondence between an
$\gamma/2$ $L^\infty$-cover of $\cG$ and an $\gamma$
Hausdorff-cover of $\cM$.

\medskip 

Taken together, the claims imply that 
$$
N(\gamma,\cM,H) \le
c_1 \binom{D}{d}^{(2 c/\kappa)^D} \exp((D-d)(2 c/\kappa)^D 2^{d/2}\gamma^{-d/2}).
$$
Taking $c_2 = 2 c$ proves the theorem.  \bigskip

\emph{Proof of Claim 1}.  
We begin by showing that (i) implies (ii).
By part 1 of Lemma \ref{lemma::smale},
each $M\in \cM$
has curvature (second fundamental form)
bounded above by $1/\kappa$.
This implies that the function identified in (i) 
has uniformly bounded second derivative and thus lies in the corresponding $\cG_{jk}$.

We prove (i) by contradiction.
Suppose that there is an $M\in\cM$
such that for every $j$ with $M\cap C_j \ne \emptyset$,
the set $M\cap C_j$ is not the graph of a single-valued mapping
for any of the $K$ coordinate frames.

Fix $j\in\{1,\ldots,J\}$.
Then in each ${\rm domain}(F_{jk})$, there is a point $u$ such that
$C_j \cap \pi_{jk}^{-1}(u \times \mathbb{R}^{D-d})$ intersects $M$
in at least two points, call them $a_k$ and $b_k$.
By construction $\norm{a_k - b_k} \le \sqrt{D-d} \cdot \kappa/c$,
and hence by choosing $c$ large enough (making the cubes small),
part 3 of Lemma \ref{lemma::smale}
tells us that $d_M(a_k,b_k) \le 2\sqrt{D-d}\kappa/c$.
Then we argue as follows:
\begin{enumerate}
\item By parts 2 and 3 of Lemma \ref{lemma::smale}
and the fact that $C_j$ has diameter $\sqrt{D} \kappa/c$
and 
$$\max_{p,q\in C_j\intersect M} \cos(\angle(T_p M,T_q M)) \ge 1 - \frac{2\sqrt{D}}{c}.$$
For large enough $c$, the maximum angle between tangent vectors
can be made smaller than $\pi/3$.

\item By part 2 of Lemma \ref{lemma::smale},
any point $z$ along a geodesic between $a_k$ and $b_k$,
$$\cos(\angle(T_{a_k} M, T_{z} M)) \ge 1 - \frac{2\sqrt{D-d}}{c}.$$
It follows that there is a point in $C_j \intersect M$
and a tangent vector $v_k$ at that point
such that $\angle(v_k, b_k - a_k) = O(1/\sqrt{c})$.

\item We have 
for each of $K = \binom{D}{d}$ coordinate frames
and associated tangent vectors $v_1, \ldots, v_K$
that are each nearly orthogonal to at least $d$ of the others.
Consequently, there are $\ge d+1$ nearly orthogonal tangent
vectors of $M$ within $C_j$.
This contradicts point 1 and proves the claim.
\end{enumerate}

\emph{Proof of Claim 2}.  
We construct the correspondence as follows. For each cube $C_j$, let $k_{j}^*$ be the smallest
$k$ such that $M\cap C_j$ is the graph of a function $\phi_{jk}\in\cG_{jk}$ as in Claim 1.
Map $M$ to $\varphi = (\phi_{1k_1^*},\ldots,\phi_{Jk_J^*})$, and let $\cF\subset\cG$ be the image of this map.
If $M \ne M' \in \cM$, then the corresponding $\varphi$ and $\varphi'$ must be distinct.
If not, then $M \cap C_j = M' \cap C_j$ for all $j$, contradicting $M \ne M'$.
The correspondence from $\cM$ to $\cF$ is thus a one-to-one correspondence.

\emph{Proof of Claim 3}.
From the results in  \cite{birman}, the set of functions defined on a 
pre-compact $d$-dimensional set that take values in a fixed dimension space
$\mathbb{R}^{m}$ with uniformly bounded second derivative has $L^\infty$ covering 
number bounded above by $c_1 e^{m(1/\gamma)^{d/2}}$ for some $c_1$.
Part 1 of Lemma \ref{lemma::smale}
shows that each $M\in \cM$
has curvature (second fundamental form)
bounded above by $1/\kappa$,
so each $\cG_{jk}$ satisfies Birman and Solomjak's conditions.
Hence, $N(\gamma, \cG_{jk}, L^{\infty}) \le c_1 e^{(D-d)(1/\gamma)^{d/2}}$.
Because all the $\cG_{jk}$'s are disjoint, simple counting arguments show that
$N(\gamma,\cG,L^{\infty}) = \left(\binom{D}{d} N(\gamma, \cG_{jk}, L^{\infty})\right)^J$,
where $J$ is the number of cubes defined above. The claim follows.
(Note that the functions in Claim 1 are defined on a subset of
${\rm domain}(F_{jk})$. But because all such functions have an extension in $\cG_{jk}$,
a covering of $\cG_{jk}$ also covers these functions defined on restricted domains.)

\emph{Proof of Claim 4}. First, note that if two functions are less than 
$\gamma$ distant in $L^\infty$, their graphs are less than $\gamma$ distant in Hausdorff distance, and vice versa.
This implies that a $\gamma$ $L^\infty$-cover of a set of functions
corresponds directly to an $\gamma$ Hausdorff-cover of the set of the functions' graphs.
Hence, in the argument that follows, we can work with functions or graphs interchangeably.

For $k\in\{1,\ldots,K\}$,
let $\cG_{jk}^\gamma$ be a minimal $L^\infty$ cover of $\cG_{jk}$ by $\gamma/2$ balls;
specifically, we assume that $\cG_{jk}^\gamma$ is the set of centers of these balls.
For each $g_{jk}\in\cG_{jk}^\gamma$,
define $f_{jk}(u) = \pi^{-1}_{jk}(u, g_{jk}(u))$.
For every $j$, choose one such $f_{jk}$,
and define a set $M' = \Union_j (C_j\cap {\rm range}(f_{jk_j}))$,
which is a union of manifolds with boundary that have curvature bounded by $1/\kappa$.
That is, such an $M'$ is piecewise smooth (smooth within each cube) but may fail to
satisfy $\Delta(M') \ge \kappa$ globally.
Let $\cA$ be the collection of $M'$ constructed this way.
There are $N(\gamma/2,\cG,L^\infty)$ elements in this collection.

By construction and Claim 2,
for each $M\in\cM$, there exists an $M'\in\cA$
such that $H(M,M') \le \gamma/2$.
In other words, the set of $\gamma/2$ Hausdorff balls around
the manifolds in $\cA$ covers $\cM$
but the elements of $\cA$ are not themselves necessarily in $\cM$.
Let $B_H(A,\gamma/2)$ denote the set of all $d$-manifolds $M\in\cM$ such that
$H(A,M) \le \gamma/2$.
Let 
\begin{equation}
\cA_0 = \Bigl\{A\in \cA:\ B_H(A,\gamma/2)\cap \cM \neq \emptyset \Bigr\}.
\end{equation}
For each $A\in \cA_0$, choose some
$\tilde{A}\in B_H(A,\gamma/2)\cap \cM$.
By the triangle inequality, the set
$\{\tilde{A}:\ A\in \cA_0\}$
forms an $\gamma$ Hausdorff-net for $\cM$.
This proves the claim.
\end{proof}

\vspace{.5cm}

We are almost ready to compute the entropy.
We will need the following lemma.

\vspace{.5cm}

\begin{lemma}
\label{lemma::volume-diff}
Let $0 < \gamma < \kappa-\sigma$.
There exists a constant $K>0$ 
(depending only on ${\cal K}, \kappa$ and $\sigma$)
such that,
for any $M_1,M_2\in {\cal M}(\kappa)$,
$H(M_1,M_2) \leq \gamma$ implies that
$|V(M_1\oplus\sigma)-V(M_2\oplus\sigma)| \leq K \gamma$.
Also, for any $M\in {\cal M}(\kappa)$,
$|V(M\oplus(\sigma+\gamma))-V(M\oplus\sigma)| \leq K \gamma$.
\end{lemma}

\begin{proof}
Let $S_j = M_j \oplus \sigma$, $j=1,2$.
Then, using (\ref{eq::fiber}),
\begin{equation}
S_2 \subset M_1 \oplus (\sigma + \gamma) = \Union_{u\in M_1} L_{\sigma+\gamma}(u).
\end{equation}
Hence, uniformly over ${\cal M}$,
$$
V(S_2) \le \int_{M_1} \nu_{D-d}(L_{\sigma+\gamma}(u)) d\mu_{M_1} \leq
\int_{M_1} \nu_{D-d}(L_\sigma(u)) d\mu_{M_1} + K \gamma = V(S_1) + K \gamma
$$
since 
$\nu_{D-d}(B(u,\sigma+\gamma)) \leq \nu_{D-d}(B(u,\sigma)) + K\gamma$
for some $K>0$ not depending on $M_1$ or $M_2$.
By a symmetric argument,
$V(S_1) \leq V(S_2) + K\gamma$.
Hence,
$|V(M_1\oplus\sigma)-V(M_2\oplus\sigma)| \leq K \gamma$.
The second statement is proved in a similar way.
\end{proof}

Now we construct a Hellinger bracketing.
Let $\gamma = \epsilon^2$.
Let $\mathbb{M}_\gamma = \{M_1,\ldots, M_N\}$
be a $\gamma$-Hausdorff net of manifolds.
Thus, by Theorem \ref{thm::hausdorff-covering-number}, $N=N(\epsilon^2, {\cal M},H) \leq c_1 e^{c_2 (1/\epsilon)^d}$.
Let $\omega$ denote the volume of a sphere of radius $\sigma$.
Let
$q_j$ be the density corresponding to $M_j$.
Define
$$
u_j(y) = 
\left(q_j(y) + \frac{2\epsilon^2}{V(M_j \oplus(\sigma+\epsilon^2))}\right)
I(y\in M_j \oplus (\sigma+\epsilon^2))
$$
and
$$
\ell_j(y) = 
\left(q_j(y) - \frac{2\epsilon^2}{V(M_j \oplus(\sigma-\epsilon^2))}\right)
I(y\in M_j \oplus (\sigma-\epsilon^2)).
$$
Let ${\cal B}=\{ (\ell_1,u_1),\ldots, (\ell_N,u_N)\}$.

\begin{lemma}
\label{lemma::this-is-a-bracket}
${\cal B}$ is an $\epsilon$-Hellinger bracketing of ${\cal Q}$.
Hence,
$\cH_{[\,]}(\epsilon,{\cal Q},h) \leq C (1/\epsilon)^d$.
\end{lemma}

\begin{proof}
Let $M\in {\cal M}(\kappa)$ and let
$Q=Q_M$ be the corresponding distribution.
Let $q$ be the density of $Q$.
$Q$ is supported on
$S=M\oplus\sigma$.
There exists $M_j\in\mathbb{M}_\gamma$ such that
$H(M,M_j) \leq \epsilon^2$.
Let $y$ be in $S$.
Then there is a $x\in M$ such that
$||y-x||\leq \sigma$.
There is a $x'\in M_j$ such that
$||x-x'|| \leq \epsilon^2$.
Hence,
$d(y,M_j) \leq \sigma+\epsilon^2$ and thus
$y$ is in the support of $u_j$.
Now, for $y\in S$,
$u_j(y) - q(y) = 2\epsilon^2/V(M_j\oplus(\sigma +\epsilon^2)) \geq 0$.
Hence,
$q(y) \leq u_j(y)$.
By a similar argument,
$\ell_j(y)\leq q(y)$.
Thus ${\cal B}$ is a bracketing.
Now
\begin{eqnarray*}
\ell_1(\ell_j,u_j)  &=&
\int u_j - \int \ell_j =
\left(1 + \frac{2 K \epsilon^2}{\omega}\right) - 
\left(1 - \frac{2 K \epsilon^2}{\omega}\right) = \frac{4 K \epsilon^2}{\omega}.
\end{eqnarray*}
Finally, by (\ref{eq::hell-l1}),
$h(u_j,\ell_j) \leq \sqrt{\ell_1(\ell_j,u_j)} = C \epsilon$.
Thus ${\cal B}$ is a $C\epsilon$-Hellinger bracketing.
\end{proof}

\vspace{.25cm}

\noindent
{\bf Step 2b. Hellinger Rate.}

\begin{lemma}
\label{lemma::2b}
Let $\tilde Q$ be the mle. Then
$$
\sup_{Q\in {\cal Q}}
Q^n \left( \left\{h(Q,\tilde Q) >  C_0 n^{-\frac{1}{d+2}} \right\}\right)\leq
\exp\left\{ - C n^{\frac{d}{2+d}}\right\}.
$$
\end{lemma}

\begin{proof}
We have shown (Lemma \ref{lemma::this-is-a-bracket}) that
$\cH_{[\,]}(\epsilon,{\cal Q},h)\leq C(1/\epsilon)^d$.
Solving the equation
$H_{[\,]}(\epsilon_n, {\cal Q},h)=n\,\epsilon_n^2$ 
from Theorem \ref{thm::shen} we get $\epsilon_n = (1/n)^{1/(d+2)}$.
From Lemma \ref{lemma::shen2}, for all $Q$
$$
Q^n \left( \left\{h(Q,\tilde Q) > C_0 n^{-\frac{1}{d+2}} \right\}\right)\leq
c_1 e^{-c_2 n \epsilon_n^2} =
\exp\left\{ - C n^{\frac{d}{2+d}}\right\}.
$$
\end{proof}

\vspace{.25cm}

\noindent
{\bf Step 2c. Relating Hellinger Distance and Hausdorff Distance.}

\begin{lemma}
\label{lemma::calibration}
Let $c = (\kappa-\sigma) \sqrt{\pi}C_*/(2\, \Gamma(D/2+1))$.
If $M_1,M_2 \in {\cal M}(\kappa)$ and
$h(Q_1,Q_2) < c$ then
$$
H(M_2,M_2) \leq
\left[\frac{ 2 }{\sqrt{\pi}}
\left( \frac{ \Gamma(D/2+1)}{C_*}\right)^{1/D}\right]
h^{\frac{1}{D}} (Q_1,Q_2)
$$

\end{lemma}

\begin{proof}
Let $b=H(M_1,M_2)$ and
$\gamma = \min\{ \kappa-\sigma,b\}$.
Let $S_1, S_2$ be the supports of $Q_1$ and $Q_2$.
Because $H(M_1,M_2) = b$,
we can find points $x \in M_1$ and $y\in M_2$
such that $\norm{y - x} = b$.
Note that $T_x M_1$ 
and $T_y M_2$.
are parallel,
otherwise we could move $x$ or $y$ and increase
$\norm{y - x}$.
It follows that the line segment $[x,y]$ is along a common
normal vector of the two manifolds and
we can write
$y = x \pm b u$ for some $u\in L_\sigma(u,M)$.
Without loss of generality, assume that
$y = x + b u$.
Let
$x' = x+ \sigma u$ and
$y' = y+ \sigma u$.
Hence,
$x'\in \partial S_1$, $y'\in \partial S_2$ 
and $||x' - y'||=b$.
Note that $\partial S_1$ and $\partial S_2$ are themselves smooth $D$-manifolds
with $\Delta(\partial S_i) \ge \kappa - \sigma > 0$.

We now make the following three claims:
\begin{enumerate}
\item $y' \in S_2 - S_1$.
\item $(x',y'] \subset S_2 - S_1$
\item $\mathop{\rm interior} B\left(\frac{x'+y'}2, \frac{\gamma}{2}\right) \subset S_2 - S_1$
\end{enumerate}

First, note that $y'$ differs from $y$ along a fiber of $M_2$ by exactly $\sigma$,
therefore $[x',y'] \subset S_2$.
Second, because $x'\in \partial S_1$, there is a neighborhood of $x'$ in $[x',y']$
that is not contained in $S_1$. 
Hence, if there is a point in $S_1 \intersect [x',y']$ there must be a point $z'\in \partial S_1 \intersect [x',y']$,
with $z' \ne x'$.
This implies the existence of two distinct points whose fibers
of length less than $\kappa - \sigma$ cross, which contradicts the fact that
$\Delta(\partial S_1) \ge \kappa - \sigma$.
Claims 1 and 2 follows.

Let $B = B\left(\frac{x'+y'}2, \frac{\gamma}2\right)$.
By construction, $B$ is tangent to $\partial S_1$ at $x'$ and tangent to $\partial S_2$ at $y'$,
and $B$ contains $[x',y']$.
The ball has radius
$\gamma/2 = (1/2) \min\{\kappa-\sigma,b\} < \kappa-\sigma$.
Because $B$ intersects $S_2 - S_1$,
the interior of $B$ cannot intersect either $\partial S_1$ or $\partial S_2$.
Claim 3 follows by a similar argument as in the proof of Claim 2.
(In particular, if there were a point in the interior of $B$ that is either
in $S_1$ or outside $S_2$, a line segment from $(x'+y')/2$ to that point
would have to intersect the corresponding boundary, which cannot happen.)

Now $V(B) = (\gamma/2)^D \pi^{D/2}/\Gamma(D/2+1)$.
So
\begin{eqnarray*}
h(Q_1,Q_2) & \geq &
\ell_1(Q_1,Q_2) = \int |q_1 - q_2| \geq
\int_{S_1\cap S_2^c} |q_1 - q_2| \\
&=&
\int_{S_1\cap S_2^c}q_1 = Q_1(S_1\cap S_2^c) \geq C_* V(S_1\cap S_2^c) =
C_*(\gamma/2)^D \pi^{D/2}/\Gamma(D/2+1).
\end{eqnarray*}
Hence,
$$
\gamma =\min\{\kappa-\sigma,b\} \leq \left[\frac{ 2 }{\sqrt{\pi}}
\left( \frac{ \Gamma(D/2+1)}{C_*}\right)^{1/D}\right]
h^{1/D}(Q_1,Q_2).
$$
If $\kappa-\sigma\leq b$ this implies that
$h(Q_1,Q_2) > c$ which 
contradicts the assumption that
$h(Q_1,Q_2) < c$.
Therefore,
$\gamma = b$ and
the conclusion follows.
\end{proof}

\noindent
{\bf Step 2d. Computing The Hausdorff Rate of the Pilot.}

\begin{lemma} \label{lemma::pilot}
Let
$a_n = \left(\frac{C_0}{n}\right)^{\frac{2}{D(d+2)}}$.
For all large $n$,
\begin{equation}
\sup_{Q\in {\cal Q}}Q^n \left( \{H(M,\tilde M) >  a_n\}\right)\leq
\exp\left\{ - C n^{\frac{d}{2+d}}\right\}.
\end{equation}
\end{lemma}

\begin{proof}
Follows by combining
Lemma \ref{lemma::2b} and
Lemma \ref{lemma::calibration}.
\end{proof}

We conclude that,
with high probability,
the true manifold $M$ is contained in the set
${\cal M}_n = \Bigl\{ M\in {\cal M}(\kappa):\ H(\tilde M, M) \leq a_n\Bigr\}$.

\vspace{1.0cm}
\noindent
\fbox{\bf Step 3:} {\bf Cover With Slabs}
\vspace{0.5cm}

Now we cover the pilot estimator $\tilde M$ with 
(possibly overlapping)
slabs.
Let
$\delta_n = 
\left( \frac{C \log n}{n}\right)^{\frac{1}{2 + d}}$.
It follows from part 6 of Lemma \ref{lemma::smale} that
there exists a collection of points $F=\{x_1,\ldots,x_N\}\subset \tilde M$, such that 
$N = (c\delta_n)^{-d} = ( C n /\log n)^{d/(2+d)}$
and such that $\tilde M \subset \Union_{j=1}^N B_D(x_j, c\delta)$.

\vspace{.25cm}

\noindent
{\bf Step 3a. The Fibers of $\tilde M$ Cover $M$ Nicely.}

\begin{lemma}
\label{lemma::toothpick}
Let $b=\sigma+a_n$.
For $\tilde x\in \tilde M$,
let $L_b(\tilde x) = T_{\tilde x}^\perp \tilde M \cap B_D(\tilde x,b)$
be a fiber at $\tilde x$ of size $b$.
Let $M\in {\cal M}_n$.
Then:
\begin{enumerate}
\item If $\tilde x \in \tilde M$ and $x\in M$
are such that $\norm{x - \tilde x} \le a_n$,
then ${\sf angle}(T_x M, T_{\tilde x} \tilde M) < \pi/4$.
\item $L_b(\tilde x) \intersect M \ne \emptyset$.
\item If $x \in L_b({\tilde x}) \intersect M$,
then $\norm{x - \tilde x} \le 2 a_n$.
\item For any $\tilde x \in \tilde M$,
$\#\{L_b({\tilde x}) \intersect M\} = 1.$
\item We have
$M \subset \Union_{\tilde x \in \tilde M} L_b({\tilde x}).$
\end{enumerate}
\end{lemma}

\begin{proof} {\em 1.}
Let $x$ and $\tilde x$ be as given in the statement of the lemma
and let $\theta = {\sf angle}(T_x M, T_{\tilde x} \tilde M)$.
Suppose that $\theta \geq \pi/4$.
There exists unit vectors $u \in T_{\tilde x} \tilde M$
and $v\in T_x M$ such that ${\sf angle}(u,v) = \theta$.
Without loss of generality, we can assume that $x = \tilde x$.
(The extension to the case $x\neq \tilde x$ is straightforward.)

Consider the plane defined by $u$ and $v$
as in Figure \ref{fig::angleplot}.
We assume, without loss of generality, that $(u+v)/2$ 
generates the $x$-axis in this plane and that $v$ lies above the 
$x$-axis and $u$ lies below the $x$ axis.
Let $\ell$ denote the horizontal line, parallel to the $x$-axis and 
lying $2 a_n$ units above the horizontal axis.
Hence, $u$ and $v$ each make an angle greater than $\pi/8$
with respect to the $x$-axis.

\begin{figure}
\begin{center}
\includegraphics[width=2.9in,angle=4]{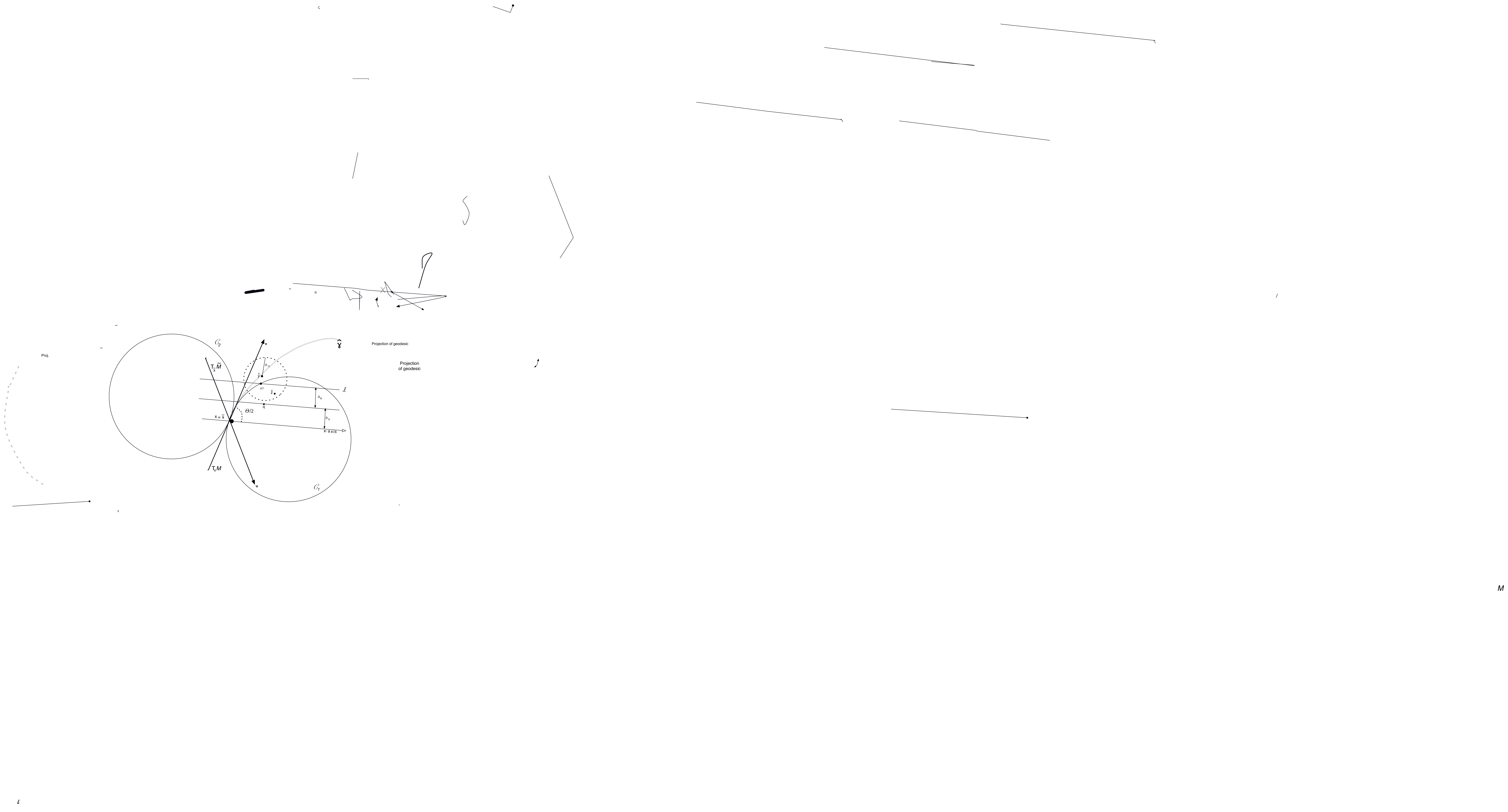}
\vspace{-1.3cm}
\caption{Figure for the proof of part 1 of Lemma \ref{lemma::toothpick}.}
\label{fig::angleplot}
\end{center}
\end{figure}

Consider the two circles $\cC_1$ and $\cC_2$
tangent to $M$ at $x$ with radius $\kappa$
where $\cC_1$ lies below $v$ and $\cC_2$ lies above $v$.
Let $w$ be the point at which $\cC_1$
intersects $\ell$.
The arclength of $\cC_1$ from $x$ to $w$ is $C a_n$ for some $C>1$.
Let $\gamma$ be 
the geodesic on $M$ through $x$
with gradient $v$.
The projection $\hat\gamma$ of $\gamma$ into the plane
must fall between $\cC_1$ and $\cC_2$.
Let $y=\gamma(C a_n)$ and $\hat y$ be the projection of $y$ into the plane.

Now 
$||y-\tilde x|| \geq ||\hat y - \tilde x|| \geq ||w - \tilde x|| \geq 2 a_n > a_n$.
There exists $\tilde z \in \tilde M$ such that
$||\tilde z - y || \leq a_n$.
Hence,
$||\hat z - \hat y || \leq a_n$
where $\hat z$ is the projection of $\tilde z$ into the plane.
Let $q$ be the point on the plane with coordinates $( a_n \sqrt{C^2-1}, a_n)$.
Thus, $||q-\tilde x|| = C\, a_n$.
Note that
$\angle(\hat z - \tilde x,u)$
is larger than 
the angle between
$q-\tilde x$ and the $x$-axis
which is
${\rm arctan}\left( \frac{1}{\sqrt{C^2-1}}\right) \equiv \alpha >0$.
Hence,
$$ 
\angle(\tilde z - \tilde x,u) \geq 
\angle(\hat z - \tilde x,u) \geq  \alpha.
$$

Let $\tilde \gamma$ be a geodesic on $\tilde M$,
parameterized by arclength connecting
$\tilde x$ and $\tilde z$.
Thus
$\tilde \gamma(0) = \tilde x$ and
$\tilde \gamma(T) = \tilde z$ for some $T$.
There exists some $0 \leq t\leq T$ such that
$\gamma'(t) \propto \tilde z - \tilde x$. So
$$
\angle(\gamma'(t),\gamma'(0)) = \alpha >0.
$$

However, $||\tilde z - \tilde x|| \leq (C+1) \, a_n$
which implies, by part 2 of 
Lemma \ref{lemma::smale}, that
$\angle(\gamma'(t),\gamma'(0)) = O(\sqrt{a_n}) < \alpha$
which is a contradiction.

\vspace{1cm}

\noindent
{\em 2.} For any $\tilde x \in \tilde M$, the closest point $x\in M$
must satisfy $\norm{x - \tilde x} \le a_n$.
Let $y$ be the projection of $x$ onto $T_{\tilde x} \tilde M$.
Let $U = T_{\tilde x} \tilde M \intersect B_d(y,  a_n)$.
Let ${\rm Cyl} = \Union_{u\in U} B_D(u, 3 a_n) \intersect \left(T_{\tilde x}\tilde M\right)^\perp$.
${\rm Cyl}$ is a small hyper-cylinder containing $y$ and $\tilde x$, with the former in the center.
$M$ cannot intersect the top or bottom faces of the cylinder.
Otherwise, we can find a point $p\in M$ such that
$\angle (T_{\tilde x}\tilde M,T_p M) > {\rm arctan}(1) = \pi/4$
contradicting {\em 1.}
Thus, any path through $x$ on $M$ must intersect the sides of ${\rm Cyl}$.
Hence, $L_b(\tilde x) \intersect M \ne \emptyset$.

\vspace{1cm}

\noindent {\em 3.} Let $x\in M\cap L_b(\tilde x)$.
Suppose that $||x-\tilde x|| > 2a_n$.
There exists $q\in \tilde{M}$ such that
$||q-x|| \leq a_n$.
Note that $||q- \tilde x|| > a_n$.
Now we apply
part 5 Lemma \ref{lemma::smale}
with
$p=\tilde x$ and $v=x$.
This implies that
$||v-p|| = ||x-\tilde x|| < a_n^2/\kappa$
which contradicts the assumption that
$||x-\tilde x|| > 2a_n$.

\vspace{1cm}

\noindent {\em 4.}
Suppose that more than one point of $M$ were in $L_b({\tilde x})$.
Pick two and call them $x_1$ and $x_2$.
By {\em 3}, $\norm{x_i - \tilde x} \le 2a_n$.
It follows that $\norm{x_1 - x_2} \le 4 a_n$
and thus they are $O(a_n)$ close in geodesic distance
by part 3 of Lemma \ref{lemma::smale}.
Hence, there is a geodesic on $M$ connecting $x_1$ and $x_2$
that is contained strictly within the $C a_n$ ball.
Because $x_2 - x_1$ lies in $L_b({\tilde x})$ and is consequently
orthogonal to $T_{\tilde x} \tilde M$,
there must exist a point on the geodesic whose angle with
$T_{\tilde x} \tilde M$ equals $\pi/2$,
contradicting part {\em 1.}

\vspace{1cm}

\noindent {\em 5.}
Because $H(\tilde M,M)\leq a_n$,
we have that $M \subset {\sf tube}(\tilde M, a_n)$.
Because $a_n < \kappa$, the fibers $L_b({\tilde x})$ partition
${\sf tube}(\tilde M, a_n)$.
Hence, each $x\in M$ must lie on one (and only one) $L_b({\tilde x})$.
\end{proof}

\vspace{.25cm}

\noindent
{\bf Step 3b. Construct slabs that cover $M$ nicely.}
Let $\gimel_j = B_D(x_j,\delta_n)\cap \tilde M$.
Define the slab
\begin{equation}
R_j = \bigcup_{x\in\gimel_j}  L_b(x,\tilde M).
\end{equation}

\begin{lemma}
\label{lemma::slabs}
The collection of slabs
$R_1,\ldots, R_N$ has the following properties.
Let $M\in {\cal M}_n$.
\begin{enumerate}
\item $M \subset \bigcup_{j=1}^N R_j$.
\item $M\cap R_j$ is function-like over $R_j$.
That is, there exists a function
$g_j:\gimel_j \to \mathbb{R}^{D-d}$ such that
$M\cap R_j= \{ g_j(x): \ x\in \gimel_j\}$.
\item For each $x\in \gimel_j$,
$L_b(x)\cap M\neq \emptyset$.
\item There exists a linear function
$\ell_j:\gimel_j \to \mathbb{R}^{D-d}$ such that
$\sup_{x\in\gimel_j} ||g_j(x) - \ell_j(x)|| \leq C \delta_n^2$.
\item $\sup_{M\in {\cal M}_n} {\sf diam}(M\cap R_j) \leq C \delta_n$.
\end{enumerate}
\end{lemma}

Thus the slabs cover $M$
and $M$ cuts across $R_j$
is a function-like way.
Moreover, $M\cap R_j$ is nearly linear.

\vspace{1cm}

\begin{proof}
The first three claims follow immediately
from Lemma \ref{lemma::toothpick}.
In particular, $g_j$ in claim {\em 2} is defined by
$g_j(x) = \{M \cap L_b(x)\}$.
Now we show {\em 4.}
We can write
$g_j(x) = g_j(x_j) + (x-x_j)^T \nabla g + \frac{1}{2} (x-x_j)^T {\sf Hess}\, (x-x_j)$
where
${\sf Hess}$ is the Hessian matrix of $g_j$ evaluated at some point between $x$ and $x_j$.
By part 1 of Lemma \ref{lemma::smale},
the largest eigenvalue of {\sf Hess} is bounded above by $1/\kappa$.
Since $||x-x_j|| \leq c \delta_n^2$, the claim follows.
Part {\em 5} follows easily.
\end{proof}

\vspace{1.0cm}
\noindent
\fbox{\bf Step 4:} {\bf Local Conditional Likelihood}
\vspace{0.5cm}

\noindent
Recall that
${\cal M}_n = \{ M\in {\cal M}(\kappa):\ H(\tilde M,M) \leq a_n\}$.
Let
\begin{equation}
{\cal Q}_n = \{ Q_M:\ M\in {\cal M}_n\}.
\end{equation}
Consider a slab $R_j$.
For each $Q\in {\cal Q}_n$ define
$Q_j \equiv Q(\cdot |R_j)$ by
$Q_j(A) = Q(A\cap R_j)/Q(R_j)$.
Note that 
$Q_j$ is supported over $\tube(M,\sigma)\cap R_j$.
Let
${\cal Q}_{n,j} =\{ Q_j:\ Q\in {\cal Q}_n\}$.
Before we proceed
we need to establish the following.

\vspace{.25cm}

\begin{lemma}
\label{lemma::thevolume}
Let ${\cal I}_j(M) = {\sf tube}(M,\sigma)\cap R_j$.
Then there exists $c_0>0$ such that
$$
\inf_{M\in {\cal M}_n} V( {\cal I}_j(M)) \geq c_0\delta_n^d.
$$
\end{lemma}

\begin{proof}
By Lemma \ref{lemma::slabs},
$M\cap R_j$ lies in a slab of size $a_n$ orthogonal to $\gimel_j$.
Because the angle between the two manifolds on this set must be no more than $\pi/4$
and because $a_n > \delta_n$,
the manifold $M$ cannot intersect both the ``top'' and ``bottom'' surfaces
of the slab.
Hence, for large enough $C>0$,
$\cJ_j = \Union_{x\in\gimel_j} B_D(x, \sigma/C) \subset \cI_j$.
By construction, $V(\cI_j) \ge V(\cJ_j)  \ge c \delta_n^d$.
\end{proof}

\vspace{.25cm}

\noindent
{\bf Step 4a. The Entropy of ${\cal Q}_{n,j}$.}

\begin{lemma} \label{lemma::entropy-distr}
$\cH_{[\,]}(\epsilon,{\cal Q}_{n,j},h) \leq c_1 \log(c_2 /\epsilon)$.
\end{lemma}

\begin{proof}
We begin by creating a $\gamma$ Hausdorff net for $\cQ_{n,j}$.
To do this, we will parameterize the support of these
distributions.
Each $Q\in\cQ_{n,j}$ 
has support in the collection 
$\cS_{n,j} = \{(M \oplus \sigma)\cap R_j :\ M\in\cM_n\}$.
We will construct a $\gamma$-Hausdorff net for $\cS_{n,j}$.

Let $\tilde x\in\tilde M$ be the center of $\gimel_j$.
Let $y_1, \ldots, y_r$ be a $c_1\gamma$-net of $L_b(\tilde x)$,
and let $\theta_1 < \theta_2 < \cdots < \theta_s < \pi/2 - \eta$
for a small, fixed $\eta > 0$
where $\theta_j - \theta_{j-1} \leq c_2 \gamma$.
Note that $r = O(\gamma^{-(D - d)})$ and $s = O(1/\gamma)$.
For every pair $y_i$ and $\theta_j$, let $M_{ij}$ be a $M\in\cM_n$
that crosses through $y_i$ with ${\sf angle}(T_{y_i} M, T_{\tilde x} \tilde M) = \theta_j$.
These manifolds comprise a collection 
of size $O((1/\gamma)^{D-d-1})$
which we will denote by ${\sf Net}(\gamma)$.

Let $M\in\cM_n$.  
Let $y$ be the point where $M$ crosses $L_b(\tilde x)$.
Let $y_i$ be the closest point in the net to $y$
and let $\theta_j$ be the closest angle in the net to
$\angle(T_{y}M, T_{\tilde x} \tilde M)$.
Because the angle between $M$ and $M_{ij}$ is strictly less than $\pi/4$
(part 1 of Lemma \ref{lemma::toothpick})
and the slab $R_j$ has radius $\delta_n$,
it follows that $H(M, M_{ij}) \le C_1 \gamma + \delta_n C_2 \gamma \le C \gamma$.
Hence, ${\sf Net}(\gamma)$ is a $\gamma$-Hausdorff net.

Now consider ${\sf Net}(\gamma)$
with $\gamma = \epsilon^2$.
For each $M_{ij}\in {\sf Net}(\gamma)$
let $q_{ij}$ be the correspondng density and
define $u_{ij}$ and $\ell_{ij}$ by
$$
u_{ij}(y) = 
\left(q_{ij}(y) + \frac{C\epsilon^2}{V(M_{ij} \oplus(\sigma+\epsilon^2))}\right)
I(y\in M_{ij} \oplus (\sigma+\epsilon^2))
$$
and
$$
\ell_{ij}(y) = 
\left(q_{ij}(y) - \frac{C\epsilon^2}{V(M_{ij} \oplus(\sigma-\epsilon^2))}\right)
I(y\in M_j \oplus (\sigma-\epsilon^2)).
$$
Let ${\cal B} = \{ (\ell_{ij},u_{ij})\}$.

Let $M\in {\cal M}_n$
and let $M_{ij}$ be the element of the net closest to $M$.
It follows easily that
$u_{ij} \geq q_M  \geq \ell_{ij}$.
Thus ${\cal B}$ is a bracketing.
Now,
$$
\int u_{ij} - \ell_{ij} =
1+C \epsilon^2 - (1 - C \epsilon^2) = 2 C \epsilon^2.
$$
Hence,
$h(u_{ij},\ell_{ij}) \leq \sqrt{\int |u_{ij} - \ell_{ij}|} = \sqrt{2C} \epsilon$.
Hence, ${\cal B}$ is an $\sqrt{2C}-\epsilon$-bracketing.
So,
\begin{equation}
\cH_{[\,]}(\epsilon, {\cal Q}_{n,j}, h)   \le  (D-d-1)  \log(c/\epsilon),
\end{equation}
which proves the lemma.
\end{proof}

\vspace{.25cm}

\noindent
{\bf Step 4b. Hellinger Rate of the Conditional MLE.}
Let $\hat q$ be the mle over
${\cal Q}_{n,j}$ using the $Y_i$'s in $R_j$.
Let $\hat M$ be the manifold
corresponding to $\hat q$ and let
$\hat M_j = \hat M \cap R_j$.

\begin{lemma}
\label{lemma::hellinger-in-a-slab}
For all $Q$, all $A>0$ and all large $n$,
\begin{eqnarray*}
Q^n \left(\left\{ h(Q,\hat Q) > \left(\frac{C_0\log n}{n}\right)^{\frac{1}{2+d}}\right\} \right)\leq  n^{-A}.
\end{eqnarray*}
\end{lemma}

\begin{proof}
Let $N_j$ be the number of observations 
from the second half of the data that are in $R_j$.
Let $\mu_j = \mathbb{E}(N_j)$ and
define
$m_n =  n^{\frac{2}{2+d}}$.
First, we claim that
$N_j \geq \mu_j/2 = O(m_n)$
for all $j$,
except on a set of probability
$e^{-c n^{2/(2+d)}}$. 
Let $\pi_j = Q(R_j)$.
By Lemma \ref{lemma::thevolume}
and Lemma \ref{lemma::meta},
$\pi_j \geq c \delta_n^d$ for some $c>0$.
Hence, $\mu_j \geq m_n$.
Note that
$\sigma^2 \equiv {\sf Var}(N_j)/n = \pi_j (1-\pi_j) \leq \pi_j$.
Let $t= \mu_j/2$.
By Bernstein's inequality,
$$
\mathbb{P}(N_j \leq \mu_j/2) =
\mathbb{P}(N_j -\mu_j \leq -\mu_j/2)\leq
\exp\left\{ - \frac{ t^2}{2n\sigma^2 + 2t/3} \right\}
\leq
\exp\left\{ - c n ^{2/(2+d)} \right\}.
$$
Hence, by the union bound, 
$$
\mathbb{P}(N_j \leq \mu_j/2\ {\rm for\ some\ }j) \leq
\frac{1}{N} \exp\left\{ - c n ^{2/(2+d)} \right\} \leq
 \exp\left\{ - c' n ^{2/(2+d)} \right\}
$$
since there are $N = O(1/\delta_n)$ slabs.
Thus we can assume that there are
at least order $m_n$
observations in each $R_j$.

Since
$\cH_{[\,]}(\epsilon,{\cal Q}_{n,j},h)\leq \log(C(1/\epsilon))$,
solving the equation
$\cH_{[\,]}(\epsilon,{\cal Q}_{n,j},h) = m_n\epsilon^2$ we get
$\epsilon_m \geq \sqrt{C \log m_n/m_n} = (\log n/n)^{ 2/(2(2+d))} = \delta_n$.
From
Lemma \ref{lemma::shen2},
we have, for all $Q\in Q_{n,j}$,
\begin{eqnarray*}
Q^n \left(\left\{ h(Q,\hat Q) > \delta_n\right\}\right)=
Q^n \left(\left\{ h(Q,\hat Q) > \epsilon_m\right\}\right)\leq
c_1 e^{-c_2 m_n \epsilon_m^2} \leq n^{-A}.
\end{eqnarray*}
\end{proof}

\noindent
{\bf Step 4c. Relating Hausdorff Distance to Hellinger Distance Within a Slab.}

\begin{lemma}
\label{lemma::hellinger-hausdorff-slab}
For each $M_1,M_2 \in {\cal M}_n$,
$H(M_1\cap R_j,M_2\cap R_j) \leq C\, h^2 (Q_{j1},Q_{j2})$.
\end{lemma}

\begin{proof}
Let $g_1$ and $g_2$ be defined as in Lemma \ref{lemma::slabs}.
There exists $x\in\gimel_j$
such that
$g_1(x)\in M_1$, 
$g_2(x)\in M_2$
and $||g_1(x) - g_2(x)||=\gamma$.
We claim there exists
$\gimel' \subset \gimel_j$ such that
$\inf_{x\in\gimel'}||g_1(x) - g_2(x)|| \geq \gamma/2$
and such that
$V(\gimel')\geq c\delta_n^d$.
This follows since
$g_1$ and $g_2$ are smooth,
they both lie in a slab of size $a_n$ around $\gimel_j$
and the angle between the tangent of $g_j(x)$
and $\gimel_j$ is bounded by $\pi/4$.

Create a modified manifold $M_2'$ such that $M_2'$ differs from
$M_1$ over $\gimel'$ by a $\gamma/2$ shift orthogonal to $\gimel_j$
and such that $M_2'$ is otherwise equal to $M_1$.
It follows that $\ell_1(M_1,M_2) \ge \ell_1(M_1,M_2')$
and $h(Q_1,Q_2) \ge h(Q_1,Q_2')$.

Every point in the support of the conditioned distributions can be written as an ordered pair
$(x,y)$ where $x\in \gimel_j$ and $y$ lies in a $d'$ ball
of radius $\sigma$.
$M_2'$ is shifted a distance of $\gamma/2$ in the direction orthogonal to $\gimel_j$.
As a result, the $\ell_1$ distance between $M_1$ and $M_2'$ 
equals the integral over $C'$ of the volume difference between
two $d'$ balls of the same radius that are shifted by $\gamma/2$
relative to each other.
This volume $\delta_n^d\gamma$.
Hence, 
$V(M_1\cap \gimel_j)\circ (M_2\cap\gimel_j) \geq \gamma \delta_n^d$.
Let
$A = \{x\in\gimel_j :\ q_1 >0, q_2 = 0\}$,
$B = \{x\in\gimel_j :\ q_1 >0, q_2 > 0\}$,
$C = \{x\in\gimel_j :\ q_1 =0, q_2 > 0\}$.
At least one of $A$ or $B$ has volume at least
$\gamma\delta_n^d/2$.
Without loss of generality, assume that it is $A$.
Then
\begin{eqnarray*}
h^2 (q_1,q_2) &=& \int (\sqrt{q_1} - \sqrt{q_2})^2 \geq
\int_A (\sqrt{q_1} - \sqrt{q_2})^2 = \int_A q_1\\
& \geq & \frac{ C_* c \delta_n^d \gamma}{\delta_n^d} = c C_* \gamma = cC_* H(M_1,M_2).
\end{eqnarray*}
\end{proof}

\vspace{.5cm}

\noindent
{\bf Step 4d. The Hausdorff Rate.}

\begin{lemma}
For any $A > 0$ there exists $C_0$ such that
$$
Q^n\left(\left\{ H(M\cap R_j,\hat M_j) >
\left(\frac{C_0\log n}{n}\right)^{\frac{2}{2+d}}\right\}\right)\leq
\frac{1}{n^A}.
$$
\end{lemma}

\begin{proof}
This follows by combining
Lemma
\ref{lemma::hellinger-hausdorff-slab} and Lemma
\ref{lemma::hellinger-in-a-slab}.
\end{proof}

\vspace{1.0cm}
\noindent
\fbox{\bf Step 5:} {\bf Final Estimator}
\vspace{0.5cm}

Now we can combine the estimators from the difference slabs.
Let $\hat M = \bigcup_{j=1}^N \hat M_j$.
Recall that the number of slabs is
$N = (c\delta_n)^{-d} = ( C n /\log n)^{d/(2+d)}$.

\vspace{1cm}

\noindent
{\bf Proof of Theorem \ref{thm::upper}.}
Choose an $A > 2/(2 + d)$.
We have:
\begin{eqnarray*}
Q^n\left(\left\{H(\hat M,M) > \left(\frac{C_0\log n}{n}\right)^{\frac{2}{2+d}}\right\}\right) & \leq &
\sum_j Q^n\left(\left\{H(\hat M_j ,M\cap R_j) > \left(\frac{C_0\log n}{n}\right)^{\frac{2}{2+d}}\right\}\right)\\
& \leq & \frac{N}{n^A}\\
&=&
\left( \frac{n}{C \log n}\right)^{\frac{1}{2 + d}} \times \frac{1}{n^A} \leq
\frac{c}{n^A}.
\end{eqnarray*}
Let
$r_n = \left(\frac{C_0\log n}{n}\right)^{2/(2+d)}.$
Since $M$ and $\hat M$ are contained 
in a compact set,
$H(\hat M,M)$ is uniformly bounded above by a constant $K_0$.
Hence,
\begin{eqnarray*}
\mathbb{E}_{Q}H(\hat M,M) &=&
\mathbb{E}_{Q}[ H(\hat M,M) I(H(\hat M,M)> r_n)] + 
\mathbb{E}_{Q}[ H(\hat M,M) I(H(\hat M,M)\leq r_n)]\\
& \leq &
K_0\, Q^n( H(\hat M,M) > r_n) + r_n\\
& \leq & \frac{c}{n^A} + r_n = O \left(\left(\frac{\log n}{n}\right)^{2/(2+d)}\right).
\end{eqnarray*}
$\blacksquare$

\section{A Simple, Consistent Estimator}
\label{sec::suboptimal} 

Here we give a practical, consistent estimator, one
that does not converge at the optimal rate.
It is a generalization of the estimator in
\cite{us::2010} and is similar to the estimator
in \cite{smale}.
Let
\begin{equation}
\hat S = \bigcup_{i=1}^n B_D(Y_i,\epsilon)
\end{equation}
and define
$\hat{\partial S} = \partial (\hat S)$,
$\hat\sigma = \max_{y\in \hat S} d(y,\hat{\partial S})$
and
\begin{equation}
\hat M = \Bigl\{ y\in\hat S:\ d(y,\hat{\partial S}) \geq \hat\sigma - 2\epsilon\Bigr\}.
\end{equation}

\begin{lemma}
\label{thm::simple}
Let $\epsilon_n = C (\log n/n)^{1/D}$
in the estimator $\hat M$.
Then
\begin{equation}
H(M,\hat M) = O\left(\frac{\log n}{n}\right)^{1/D}
\end{equation}
almost surely for all large $n$.
\end{lemma}

Before proving the lemma we need a few definitions.
Following \cite{cuevas-boundary}, we say that a set $S$ is 
{\em $(\chi,\lambda)$-standard} if there exist positive numbers
$\chi$ and $\lambda$ such that
\begin{equation}
\nu_D(B_D(y,\epsilon)\cap S) \geq \chi \ \nu_D(B(y,\epsilon)) \ \ \ \ \ 
{\rm for \ all\ }y\in S,\  0< \epsilon \leq \lambda.
\end{equation}
We say that $S$ is {\em partly expandable} 
if there exist $r>0$ and $R\geq 1$ such that 
$H(\partial S, \partial (S\oplus\epsilon)) \leq R \epsilon$
for all $0\leq \epsilon < r$. 
A standard set has no sharp peaks while a partly expandable 
set has not deep inlets.

\begin{lemma}
\label{lemma::standard}
If $\sigma < \Delta(M)$ then 
$S = M\oplus \sigma$ is standard
with $\chi=2^{-D}$ and $\lambda=\sigma$
and partly expandable with $r=\Delta(M)-\sigma$ and $R=1$.
\end{lemma}

\begin{proof}
Let $\chi=2^{-D}$.
Let $y$ be a point in $S$ and let $\Lambda(y) \leq \sigma$ be its distance 
from the boundary $\partial S$. If $\Lambda(y) \geq \epsilon$ then
$B_D(y,\epsilon)\cap S =B_D(y,\epsilon)$ so that
$\nu_D(B_D(y,\epsilon)\cap S) =\nu_D (B_D(y,\epsilon)) \geq
\chi \,\nu_D(B_D(y,\epsilon))$.

Suppose that $\Lambda(y) < \epsilon$. 
Let $v$ be a point on the manifold
closest to $y$ and let $y^*$ be the point on the segment joining $y$ to $v$ such 
that $||y-y^*||=\epsilon/2$. The ball $A=B_D(y^*,\epsilon/2)$ is contained in both 
$B_D(y,\epsilon)$ and $S$. Hence,
$\nu_D(B_D(y,\epsilon)\cap S) \geq \nu_D (A) \geq
 \chi \nu_D(B_D(y,\epsilon))$.
This is true for all $\epsilon \leq \sigma$, hence $S$ is $(\chi, \lambda)$-standard for
$\chi = 1/2^{D}$ and $\lambda=\sigma$.

Now we show that $S$ is partly expandable. By Proposition 1 in \cite{cuevas-boundary}
it suffices to show that a ball of radius $r$ rolls freely outside $S$ for some $r$,
meaning that, for each $y\in \partial S$, there is an $a$ such that
$y\in B(a,r)\subset \overline{S^c}$, where $S^c$ is the complement of $S$.
Let $O_y$ be the ball of radius $\Delta-\sigma$ tangent to $y$ such that
$O_y\subset S^c$. Such a ball exists by virtue of the 
fact that $\sigma < \Delta(M)$.
\end{proof}

\begin{theorem}[\cite{cuevas-boundary}]
\label{lemma::supp}
Let $Y_1,\ldots, Y_n$ be a random sample from a distribution with support $S$.
Let $S$ be compact, $(\lambda,\chi)$-standard and partly expandable.
Let 
\begin{equation}
\hat{S} = \bigcup_{i=1}^n B(Y_i,\epsilon_n)
\end{equation}
and let $\hat{\partial S}$ be the boundary of $\hat S$.
Let $\epsilon_n = C (\log n /n)^{1/D}$ with
$C > (2/(\chi\  \omega_D))^{1/D}$ 
where $\omega_D = V(B_D(0,1))$.
Then, with probability one,
\begin{equation}
H(S,\hat{S}) \leq C \left(\frac{\log n}{n}\right)^{1/D}
\ \ \ {\rm and}\ \ \ \ 
H(\partial S,\hat{\partial S}) \leq C \left(\frac{\log n}{n}\right)^{1/D}
\end{equation}
for all large $n$.
Also, $S\subset \hat S$ almost surely for all large $n$.
\end{theorem}

{\bf Proof of Lemma \ref{thm::simple}.}
Theorem \ref{lemma::supp} and 
Lemma \ref{lemma::standard} imply that
$H(S,\hat S) \leq C (\log n /n)^{1/D}$ and
$H(\partial S,\hat {\partial S}) \leq C (\log n /n)^{1/D}$.
It follows that $\hat\sigma \geq \sigma - \epsilon$.
First we show that $y\in\hat{M}$ implies that $d(y,M) \leq 4\epsilon$.
Let $y\in\hat{M}$. Then
$d(y,\partial S)  \geq  
d(y,\hat{\partial S}) -\epsilon \geq \hat\sigma - 2\epsilon-\epsilon \geq
\sigma - \epsilon - 2\epsilon - \epsilon = \sigma - 4\epsilon$.
So
$d(y,M) = \sigma - d(y,\partial S) \leq 
\sigma - \sigma + 4\epsilon  = 4 \epsilon$.
Now we show that $M\subset\hat M$. Suppose that $y\in M$. Then,
$$
d(y,\hat{\partial S})  \geq  
d(y,\partial S) -\epsilon =\sigma - \epsilon \geq
\hat\sigma - 2\epsilon
$$
so that $y\in\hat{M}$.
$\blacksquare$

\section{Conclusion and Open Questions}

We have established that the optimal rate 
for estimating a smooth manifold in Hausdorff distance
is
$n^{-\frac{2}{2+d}}$.
We conclude with some comments and open questions.

\vspace{-.15cm}

\begin{enum}
\item We have assumed that the noise is perpendicular to the manifold.
In current work we are deriving the minimax rate
under the more general assumption that $\epsilon$ is drawn from
a general, spherically symmetric distribution.
We also allow the distribution along the manifold to be any smooth density
bounded away from 0.
The rates are quite different and the methods for proving the rates
are substantially more involved.
Moreover, the rates depends on the behavior of the noise density
near the boundary of its support.
We will report on this elsewhere.
\item Perhaps the most important open question is to find
a computationally tractable estimator that
achieves the optimal rate.
It is possible that combining the estimator
in Section \ref{sec::suboptimal}
with one of the estimators in the computational geometry
literature 
(\cite{Dey})
could work.
However, it appears that some modification of such
an estimator is needed.
This is a difficult question which we hope to address in the future.
\item 
It is interesting to note that
\cite{smale}
have a Gaussian noise distribution.
While it is possible to infer the homology of $M$
with Gaussian noise it is not possible to infer
$M$ itself with any accuracy.
The reason is that manifold estimation is 
similar to (and in fact, more difficult than) nonparametric regression with 
measurement error.
In that case, it is well known
that the fastest possible rates under Gaussian noise are
logarithmic.
This highlights an important distinction
between estimating the topological structure of $M$ versus estimating $M$ in Hausdorff distance.
\item The current results take $\Delta(M)$, $d$ and $\sigma$ as known
(or at least bounded by known
constants). In practice these must be estimated.
We do not know whether there exist minimax estimators that are adaptive over
$d, \Delta(M)$ and $\sigma$.
\end{enum}

\section*{Acknowledgments}

The authors thank Don Sheehy for
helpful comments on an earlier draft of this paper.
The authors also thank the reviewers for their
comments and questions.

\section{Appendix}

\subsection{Proof of Equation \ref{eq::affinity-product}}
\label{sec::app1}

We will use the following two results
(see Section 2.4 of \cite{Tsybakov}):
\begin{equation}
h^2(P^n,Q^n) = 2 \left( 1 - \left[1 - \frac{h^2(P,Q)}{2}\right]^n\right)
\end{equation}
and
\begin{equation}
P\wedge Q \geq \frac{1}{2}\left(1 - \frac{h^2(P,Q)}{2}\right)^2.
\end{equation}
We have
\begin{eqnarray*}
P^n \wedge Q^n & \geq &
\frac{1}{2}\left(1 - \frac{h^2(P^n,Q^n)}{2}\right)^2 =
\frac{1}{8} \left( 1-\frac{h^2(P,Q)}{2}\right)^{2n}\\
& \geq &
\frac{1}{8} \left( 1-\frac{\ell_1(P,Q)}{2}\right)^{2n}
\end{eqnarray*}
since
$h^2(P,Q)\leq \ell_1(P,Q)$.

\subsection{Proof of Theorem \ref{thm::geometric} }
\label{sec::flying-saucer}

We define two manifolds
$M_1$ and $M_2$ 
with corresponding distributions $Q_1$ and $Q_2$
such that
(i) $\Delta(M_i) \geq \kappa$ $i=1,2$,
(ii) $H(M_1,M_2)=\gamma$ and 
(iii) such that
the volume of $S_1\circ S_2$ is of order $\gamma^{\frac d2 + 1}$,
where $S_i$ is the support of $Q_i$.

We write a generic $D$-dimensional vector as 
$y=(u, v, z)$, with $u \in \mathbb{R}^d$, $v \in \mathbb{R}$, $z \in \mathbb{R}^{D-d-1}$.
For each $u \in \mathbb{R}^d$ with $||u|| \leq 1$, 
define the disk in $\mathbb{R}^{d+1}$ 
$$
D_0 = \Bigl\{ (u,0) \in \mathbb{R}^{d+1}:\  u \in B_d(0, 1) \Bigr\}
$$
and let 
$$
F_0 = \partial \left(\bigcup_{(u,v) \in D_0} B_{d+1}((u,v), \kappa) \right).
$$
Now define the following $d$-dimensional manifold in $\mathbb{R}^D$
\begin{eqnarray*}
M_0
&=& 
\Bigl\{ (u,v,  0_{D-d-1}):\  (u,v) \in F_0 \Bigr\} \\
&=& 
\Bigl\{ (u, a(u),  0_{D-d-1}):\  u \in B_{d}(0,1 + \kappa) \Bigr\}
\cup
\Bigl\{ (u,-a(u),  0_{D-d-1}):\  u \in B_{d}(0,1 + \kappa) \Bigr\}
\end{eqnarray*}
where
$$
a(u) =
\left\{
\begin{array}{ll}
\kappa & {\rm if }  \;||u|| \leq 1 \\
\sqrt{\kappa^2 - (||u|| - 1 )^2}
 & {\rm if } \; 1 < \ ||u|| \leq 1 + \kappa .
\end{array}
\right.
$$
The manifold $M_0$ has no boundary and, by construction, 
$\Delta(M_0) \geq \kappa$.

Now define a second manifold that coincides with $M_0$ but has a small perturbation:
\begin{eqnarray*}
M_1
&=& 
\Bigl\{ (u, b(u),  0_{D-d-1}):\  u \in B_{d}(0,1 + \kappa) \Bigr\}
\cup
\Bigl\{ (u,-a(u),  0_{D-d-1}):\  u \in B_{d}(0,1 + \kappa) \Bigr\}
\end{eqnarray*}
where
$$
b(u) =
\left\{
\begin{array}{ll}
\gamma + \sqrt{\kappa^2 - ||u||^2}
 & {\rm if }  \;||u|| \leq \frac 12 \sqrt{4\gamma \kappa-\gamma^2} \\
2\kappa - \sqrt{\kappa^2 - (||u|| - \sqrt{4\gamma \kappa-\gamma^2} )^2}
 & {\rm if } \;\frac 12 \sqrt{4\gamma \kappa-\gamma^2} < \ ||u|| \leq \sqrt{4\gamma \kappa-\gamma^2}\\
a(u)
 & {\rm if } \; \sqrt{4\gamma \kappa-\gamma^2} < \ ||u|| \leq \sqrt{4\gamma \kappa-\gamma^2} + \kappa.\\
\end{array}
\right.
$$
Note that
$\Delta(M_1)\geq\kappa$ 
since the perturbation is obtained using portions of spheres
of radius $\kappa$.
In fact
\begin{itemize}
\item
for $||u|| \leq \frac 12\sqrt{4\gamma \kappa-\gamma^2}$, $b(u)$ is the 
$d+1$-th coordinate of the
``upper'' portion of the $(d+1)$-dimensional sphere
with radius $\kappa$ centered at $(0, \cdots, 0, \gamma)$, hence $b(u)$ satisfies
$$
||u||^2 + (b(u) -\gamma )^2 = \kappa^2  
\qquad \mbox{with } b(u) \geq \gamma;
$$
\item
for $\frac 12 \sqrt{4\gamma \kappa-\gamma^2} < \ ||u|| \leq \sqrt{4\gamma \kappa-\gamma^2}$,
$b(u)$ is the 
$(d+1)$-th coordinate of the
``lower'' portion of the $(d+1)$-dimensional sphere
with radius $\kappa$ centered at $( u \cdot\sqrt{4\gamma \kappa-\gamma^2}/||u||, 2\kappa)$
(note that the center of the sphere differs according to the direction of $u$), hence $b(u)$ satisfies
$$
\left|\left| u - \frac{u}{||u||} \sqrt{4\gamma \kappa-\gamma^2}\right|\right|^2 + (b(u) -2\kappa)^2 = \kappa^2 
\qquad \mbox{with } b(u) \leq 2\kappa.
$$
\end{itemize}

To summarize,
$M_0$ and $M_1$ are both manifolds with no boundary,
$\Delta(M_0)\geq\kappa$ 
and $\Delta(M_1)\geq\kappa$.
See Figure \ref{fig::proofTh5}.
Now
\begin{eqnarray*}
E_0 &=& M_0 - M_1 = \Bigl\{ (u,a(u),  0_{D-d-1}):\  u \in B_d(0, \sqrt{4\gamma \kappa-\gamma^2}) \Bigr\}
\\
E_1 &=& M_1 - M_0 = \Bigl\{ (u,b(u),  0_{D-d-1}):\  u \in B_d(0, \sqrt{4\gamma \kappa-\gamma^2}) \Bigr\}.
\end{eqnarray*}

\begin{figure}[h]
\begin{center}
\includegraphics[width=3.9in]{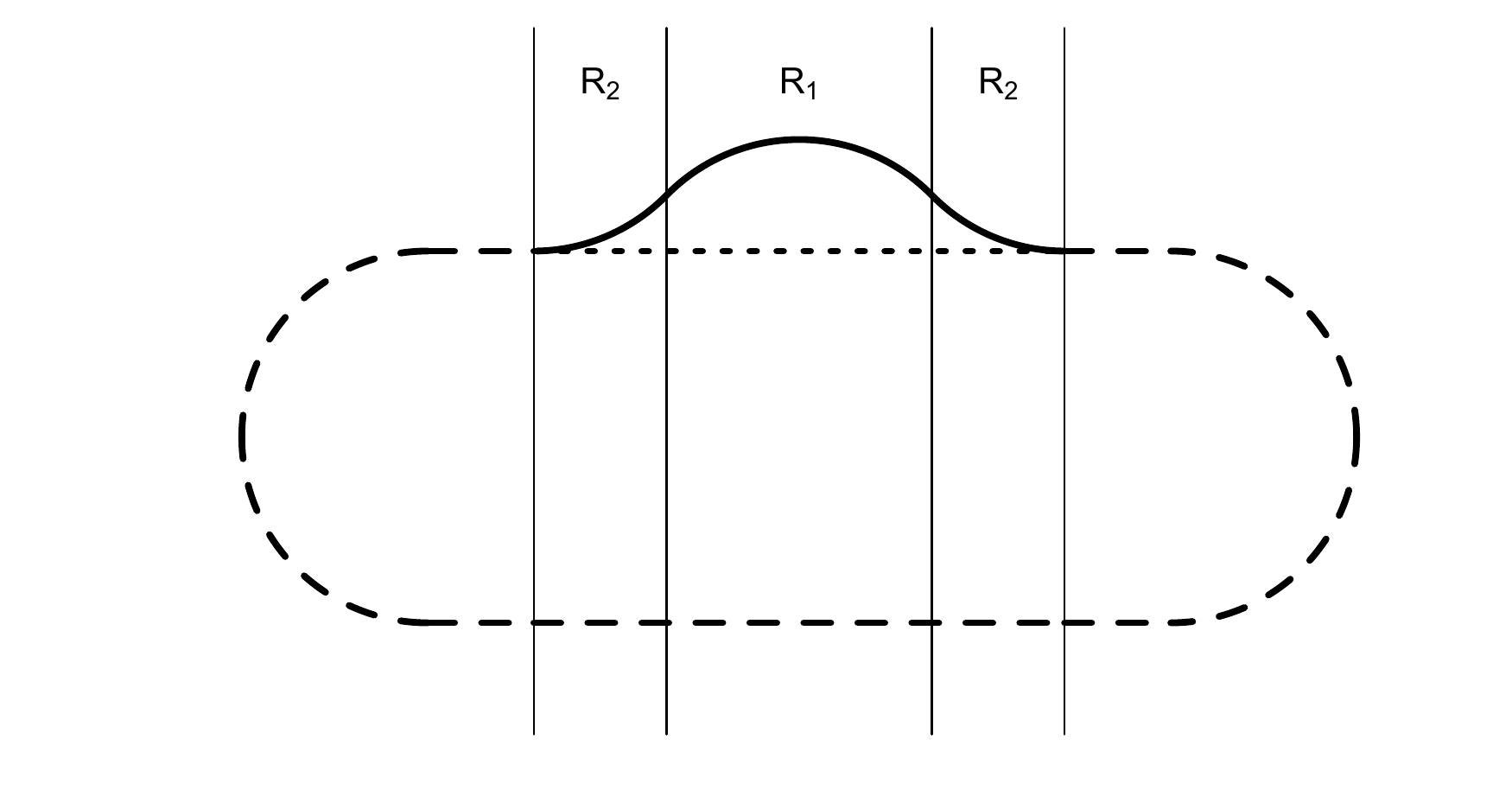}
\caption{One section of manifolds $M_0$ and $M_1$. 
The common part is dashed, $E_0$ is dotted and $E_1$ solid.
$R_1$ and $R_2$ denote the regions where the different definitions of the perturbation apply:
$R_1$ is $||u|| \leq \frac 12\sqrt{4\gamma \kappa-\gamma^2}$ while $R_2$ denotes
$\frac 12 \sqrt{4\gamma \kappa-\gamma^2} < ||u|| \leq \sqrt{4\gamma \kappa-\gamma^2}$. }
\label{fig::proofTh5}
\end{center}
\end{figure}

Note that for each point $y \in  E_0$ there exists $y' \in E_1$ such that
$||y-y'|| \leq |a(u) - b(u)| \leq \gamma$. 
Also, $y_0 = (0, a(0), 0) \in M_0$ has as its closest $M_1$ point $y_1= (0, b(0), 0)$, so that $||y_0 - x_0||=\gamma$.
Hence $H(M_0, M_1) = H(E_0, E_1) = \gamma$.

To find an upper bound for $V(S_0 \circ S_1)$, we show that each $y=(u,v,z) \in S_1 - S_0$ 
satisfies the following
conditions:
\begin{itemize}
\item[(i)]
$u \in B_d(0, \sqrt{4\gamma \kappa-\gamma^2})$;
\item[(ii)]
$z \in B_{D-d-1}(0, \sigma)$;
\item[(iii)]
$\kappa +\sigma - ||z|| < v \leq \kappa + \gamma +\sigma - ||z||]$.
\end{itemize}

If $y=(u,v,z)$ belongs to $S_1$ and has $||u|| > \sqrt{4\gamma \kappa-\gamma^2}$,
then there is a point of $M_0 \cap M_1$ within distance $\sigma$, hence $y \not \in S_1 - S_0$.
This proves (i).
Before proving (ii) and (iii), note that if $u \in B_d(0, \sqrt{4\gamma \kappa-\gamma^2})$ then
$$
\kappa = a(u) \leq b(u) \leq \kappa + \gamma.
$$
Now, let $y'=(u', b(u'), 0) \in E_1$ be the point in $S_1$ closest to $y$. We have
$$
d(y,S_1)= ||y-y'|| =  ||u-u'|| + |v-b(u')|  + ||z|| \leq \sigma.
$$
This gives condition (ii) above $||z|| \leq \sigma$ and also
\begin{equation} \label{eqn::condit}
|v-b(u')| \leq \sigma - ||z||.
\end{equation}
Since $b(u') \leq \kappa + \gamma$, we obtain
$$
v \leq b(u') +  \sigma - ||z|| \leq \kappa + \gamma + \sigma - ||z||
$$
which is the right inequality in (iii).
Finally,  
$$
\sigma < d(y, M_0) \leq || y - (u, a(u), 0) || = |v-a(u)| + ||z||
$$
which implies either $v < a(u) - (\sigma - ||z||)$ or $v > a(u) + (\sigma - ||z||)$.
The former inequality would imply 
$$
v < a(u) - (\sigma - ||z||) = \kappa - (\sigma - ||z||) \leq \inf_{u'} b(u') - (\sigma - ||z||)
$$
so that $|v-b(u')| > \sigma - ||z||$ for all $u'$, which is in contadiction with (\ref{eqn::condit}).
Hence we have $v > a(u) + (\sigma - ||z||)= \kappa + (\sigma - ||z||)$ that is  
the left inequality in (iii). 

As a consequence, 
$$
S_1 - S_0 \subset B_d(0, \sqrt{4\gamma \kappa-\gamma^2}) \times
\Bigl\{ (v,z) \in  \mathbb{R}^{D-d}:\ \kappa - \gamma +\sigma - ||z|| < v \leq \kappa + \gamma +\sigma - ||z||], 
z \in B_{D-d-1}(0, \sigma)  \Bigr\} 
$$
and
$$
V(S_0-S_1) \leq C \cdot (\sqrt{4\gamma \kappa-\gamma^2})^{d} \cdot \gamma \cdot \sigma^{D-d-1}.
$$
Hence, $V(S_0-S_1) = O(\gamma^{\frac d2 + 1})$.

With similar arguments one can show that $V(S_1 - S_0) = O(\gamma^{\frac d2 + 1})$ 
so that 
$$
V(S_0 \circ S_1) = O(\gamma^{\frac d2 + 1}).
$$
It then follows that
$\int|q_0-q_1| = O(\gamma^{(d+2)/2})$.

\bibliography{manifolds}

\begin{thebibliography}{19}
\providecommand{\natexlab}[1]{#1}
\providecommand{\url}[1]{\texttt{#1}}
\expandafter\ifx\csname urlstyle\endcsname\relax
  \providecommand{\doi}[1]{doi: #1}\else
  \providecommand{\doi}{doi: \begingroup \urlstyle{rm}\Url}\fi

\bibitem[Baraniuk and Wakin(2007)]{baraniuk}
Richard~G. Baraniuk and Michael~B. Wakin.
\newblock Random projections of smooth manifolds.
\newblock \emph{Foundations of Computational Mathematics}, 9:\penalty0 51--77,
  2007.

\bibitem[Birman and Solomjak(1967)]{birman}
M.~Birman and M.~Solomjak.
\newblock Piecewise-polynomial approximation of functions of the classes $w_p$.
\newblock \emph{Mathematics of USSR Sbornik}, 73:\penalty0 295--317, 1967.

\bibitem[Boissonnat and Ghosh(2010)]{boissonnatghosh}
Jean-Daniel Boissonnat and Arijit Ghosh.
\newblock Manifold reconstruction using tangential delaunay complexes.
\newblock In \emph{Proceedings of the 2010 annual symposium on computational
  geometry}, pages 324--333. ACM, 2010.

\bibitem[Chazal and Lieutier(2008)]{chazal2008}
Frederic Chazal and Andre Lieutier.
\newblock Smooth manifold reconstruction from noisy and non-uniform
  approximation with guarantees.
\newblock \emph{Computational Geometry}, 40:\penalty0 156--170, 2008.

\bibitem[Cheng and Dey(2005)]{chengdey}
Siu-Wing Cheng and Tamal Dey.
\newblock Manifold reconstruction from point samples.
\newblock In \emph{Proceedings of the sixteenth annual ACM-SIAM symposium on
  discrete algorithms}, pages 1018--1027. SIAM, 2005.

\bibitem[Cuevas and Rodr\'{i}guez-Casal(2004)]{cuevas-boundary}
Antonio Cuevas and Alberto Rodr\'{i}guez-Casal.
\newblock On boundary estimation.
\newblock \emph{Advances in Applied Probability}, 36\penalty0 (2):\penalty0
  340--354, 2004.

\bibitem[Devroye and Wise(1980)]{dw}
Luc Devroye and Gary~L. Wise.
\newblock Detection of abnormal behavior via nonparametric estimation of the
  support.
\newblock \emph{SIAM Journal on Applied Mathematics}, 38:\penalty0 480--488,
  1980.

\bibitem[Dey(2006)]{Dey}
Tamal Dey.
\newblock \emph{Curve and Surface Reconstruction: Algorithms with Mathematical
  Analysis}.
\newblock Cambridge University Press, 2006.

\bibitem[Dey and Goswami(2004)]{deygoswami}
Tamal Dey and Samrat Goswami.
\newblock Provable surface reconstruction from noisy samples.
\newblock In \emph{Proceedings of the twentieth annual symposium on
  computational geometry}, pages 330--339. ACM, 2004.

\bibitem[Federer(1959)]{federer}
Herbert Federer.
\newblock Curvature measures.
\newblock \emph{Transactions of the American Statistical Society}, 93:\penalty0
  418--491, 1959.

\bibitem[Genovese et~al.(2010)Genovese, Perone-Pacifico, Verdinelli, and
  Wasserman]{us::2010}
Christopher~R. Genovese, Marco Perone-Pacifico, Isabella Verdinelli, and Larry
  Wasserman.
\newblock Nonparametric filament estimation.
\newblock \emph{arXiv:1003.5536}, 2010.

\bibitem[Gonzalez and Maddocks(1999)]{gonzalez}
Oscar Gonzalez and John~H. Maddocks.
\newblock Global curvature, thickness, and the ideal shapes of knots.
\newblock \emph{Proceedings of the National Academy of Sciences}, 96\penalty0
  (9):\penalty0 4769--4773, 1999.

\bibitem[LeCam(1973)]{lecam}
L.~LeCam.
\newblock Convergence of estimates under dimensionality restrictions.
\newblock \emph{The Annals of Statistics}, pages 38--53, 1973.

\bibitem[Lee(2002)]{Lee2002}
J.M. Lee.
\newblock \emph{Introduction to Smooth Manifolds}.
\newblock Springer, 2002.

\bibitem[Niyogi et~al.(2006)Niyogi, Smale, and Weinberger]{smale}
Partha Niyogi, Steven Smale, and Shmuel Weinberger.
\newblock Finding the homology of submanifolds with high confidence from random
  samples.
\newblock \emph{Discrete and Computational Geometry}, 39:\penalty0 419--441,
  2006.

\bibitem[Niyogi et~al.(2008)Niyogi, Smale, and Weinberger]{NSW2008}
Partha Niyogi, Steven Smale, and Shmuel Weinberger.
\newblock A topological view of unsupervised learning from noisy data.
\newblock \emph{Unpublished technical report, University of Chicago}, 2008.

\bibitem[Shen and Wong(1995)]{shen}
Xiaotong Shen and Wing Wong.
\newblock Probability inequalities for likelihood ratios and convergence rates
  of sieve mles.
\newblock \emph{The Annals of Statistics}, 23:\penalty0 339--362, 1995.

\bibitem[Tsybakov(2008)]{Tsybakov}
Alexandre Tsybakov.
\newblock \emph{Introduction to Nonparametric Estimation}.
\newblock Springer, 2008.

\bibitem[Yu(1997)]{binyu}
Bin Yu.
\newblock Assouad, {F}ano, and {L}e {C}am.
\newblock In \emph{Festschrift for Lucien Le Cam}. Springer, 1997.

\end{thebibliography}

\end{document}